\documentclass{article}
\usepackage{ijcai23}

\usepackage{times}
\usepackage{soul}
\usepackage{url}
\usepackage[hidelinks]{hyperref}
\usepackage[utf8]{inputenc}
\usepackage[small]{caption}
\usepackage{graphicx}
\usepackage{amsmath}
\usepackage{amsthm}
\usepackage{booktabs}
\usepackage{algorithm}
\usepackage{algorithmic}
\usepackage[switch]{lineno}

\usepackage{amssymb}
\usepackage{mathtools}
\usepackage{microtype}
\usepackage{subfigure}
\usepackage{multirow}
\usepackage{diagbox}
\usepackage{alphalph}

\theoremstyle{plain}

\newtheorem{theorem}{Theorem}

\newtheorem{lemma}[theorem]{Lemma}

\theoremstyle{definition}
\newtheorem{definition}{Definition}

\theoremstyle{remark}

\title{On the Reuse Bias in Off-Policy Reinforcement Learning}

\author{
   Chengyang Ying$^1$
   \and
   Zhongkai Hao$^{1}$ 
   \and
   Xinning Zhou$^1$
   \and
   Hang Su$^{1,2}$
   \and
   Dong Yan$^1$
   \and
   Jun Zhu\thanks{Corresponding author.}$^{1,2}$
\affiliations
   $^{1}$Department of Computer Science \& Technology, Institute for AI, BNRist Center,\\
   Tsinghua-Bosch Joint ML Center, THBI Lab, Tsinghua University\\
   $^{2}$Pazhou Laboratory (Huangpu), Guangzhou, China\\
\emails
ycy21@mails.tsinghua.edu.cn,~~ dcszj@tsinghua.edu.cn
}

\begin{document}

\maketitle

\begin{abstract}
Importance sampling (IS) is a popular technique in off-policy evaluation, which re-weights the return of trajectories in the replay buffer to boost sample efficiency.
However, training with IS can be unstable and previous attempts to address this issue mainly focus on analyzing the variance of IS. In this paper, we reveal that the instability is also related to a new notion of Reuse Bias of IS --- the bias in off-policy evaluation caused by the reuse of the replay buffer for evaluation and optimization. 
We theoretically show that the off-policy evaluation and optimization of the current policy with the data from the replay buffer result in an overestimation of the objective, which may cause an erroneous gradient update and degenerate the performance. 
We further provide a high-probability upper bound of the Reuse Bias and show that controlling one term of the upper bound can control the Reuse Bias by introducing the concept of stability for off-policy algorithms. 
Based on these analyses, we present a novel yet simple Bias-Regularized Importance Sampling (BIRIS) framework along with practical algorithms, which can alleviate the negative impact of the Reuse Bias, and show that our BIRIS can significantly reduce the Reuse Bias empirically. 
Moreover, extensive experimental results show that our BIRIS-based methods can significantly improve the sample efficiency on a series of continuous control tasks in MuJoCo.
\end{abstract}

\section{Introduction}
Off-policy reinforcement learning algorithms~\cite{atari_nature,ddpg,sac,td3} typically have high sample efficiency and are suitable for many real-world applications.
The key idea of off-policy algorithms is to reuse historical trajectories in the replay buffer. Though promising, it also causes an issue that these trajectories are not sampled from the current policy. When estimating the expected cumulative return of the current policy with trajectories generated by historical policies, off-policy evaluation, as the core of off-policy algorithms, yields the~\emph{distribution shift} in consequence~\cite{hanna2019importance}.  
To alleviate this issue, the standard method in off-policy evaluation is importance sampling~\cite{kahn1953methods,precup2000eligibility}, which is
a widely used Monte Carlo technique to evaluate the expected return of the target policy when the training data is generated by a different behavior policy. In practice, it usually requires computing the degree of deviation between the target policy and the behavior policy. 

\begin{figure*}[htbp]
\subfigure[Reuse Bias due to resample $\tau_3$ and our BIRIS]{
\begin{minipage}[t]{0.44\linewidth}
\centering
\includegraphics[height=4.2cm,width=6.4cm]{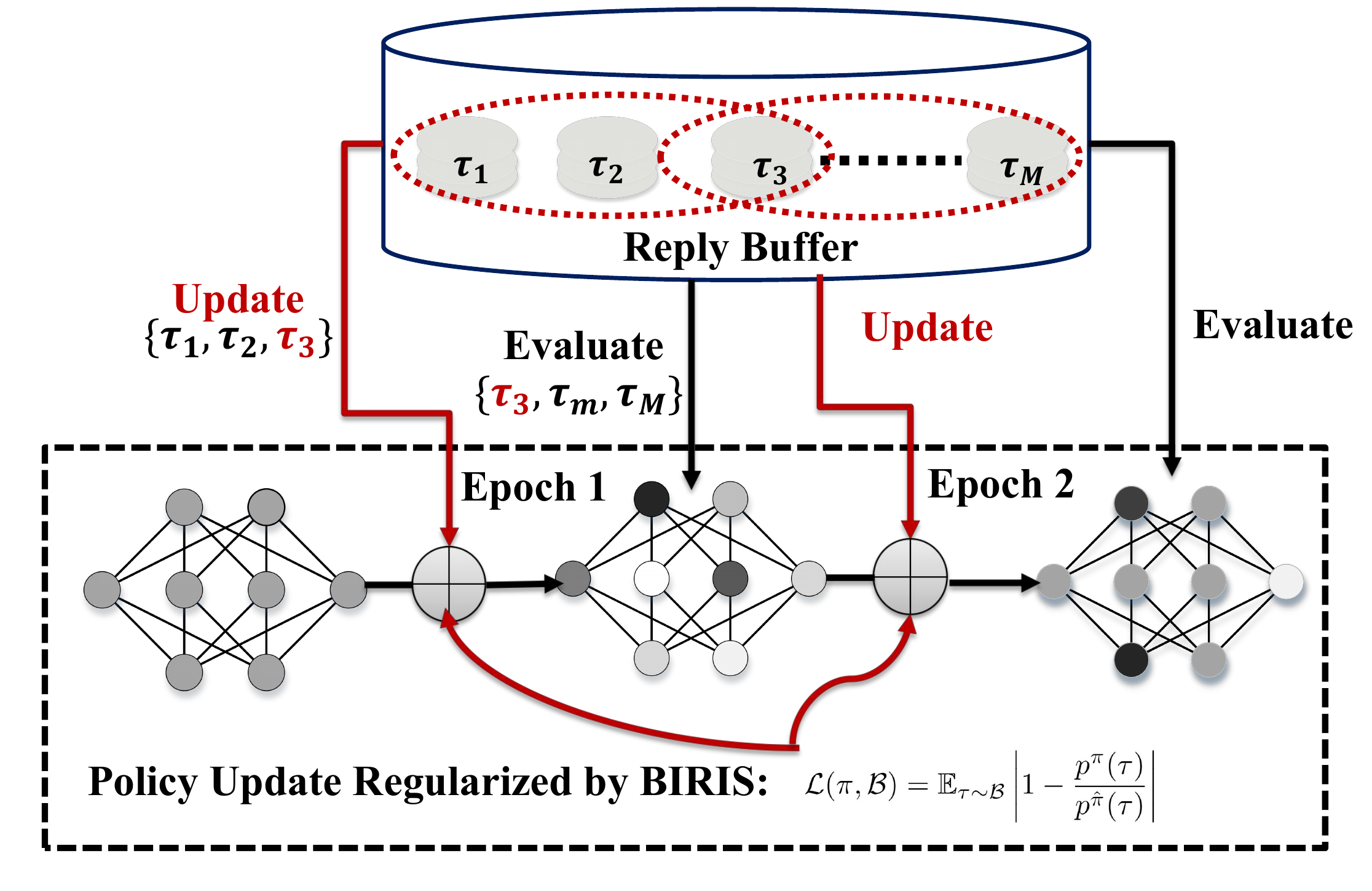}
\label{highlight_fig}
\end{minipage}}
\subfigure[Reuse Bias for PG and PG+BIRIS in MiniGrid]{
\begin{minipage}[t]{0.55\linewidth}
\centering
\includegraphics[height=4.2cm,width=8.05cm]{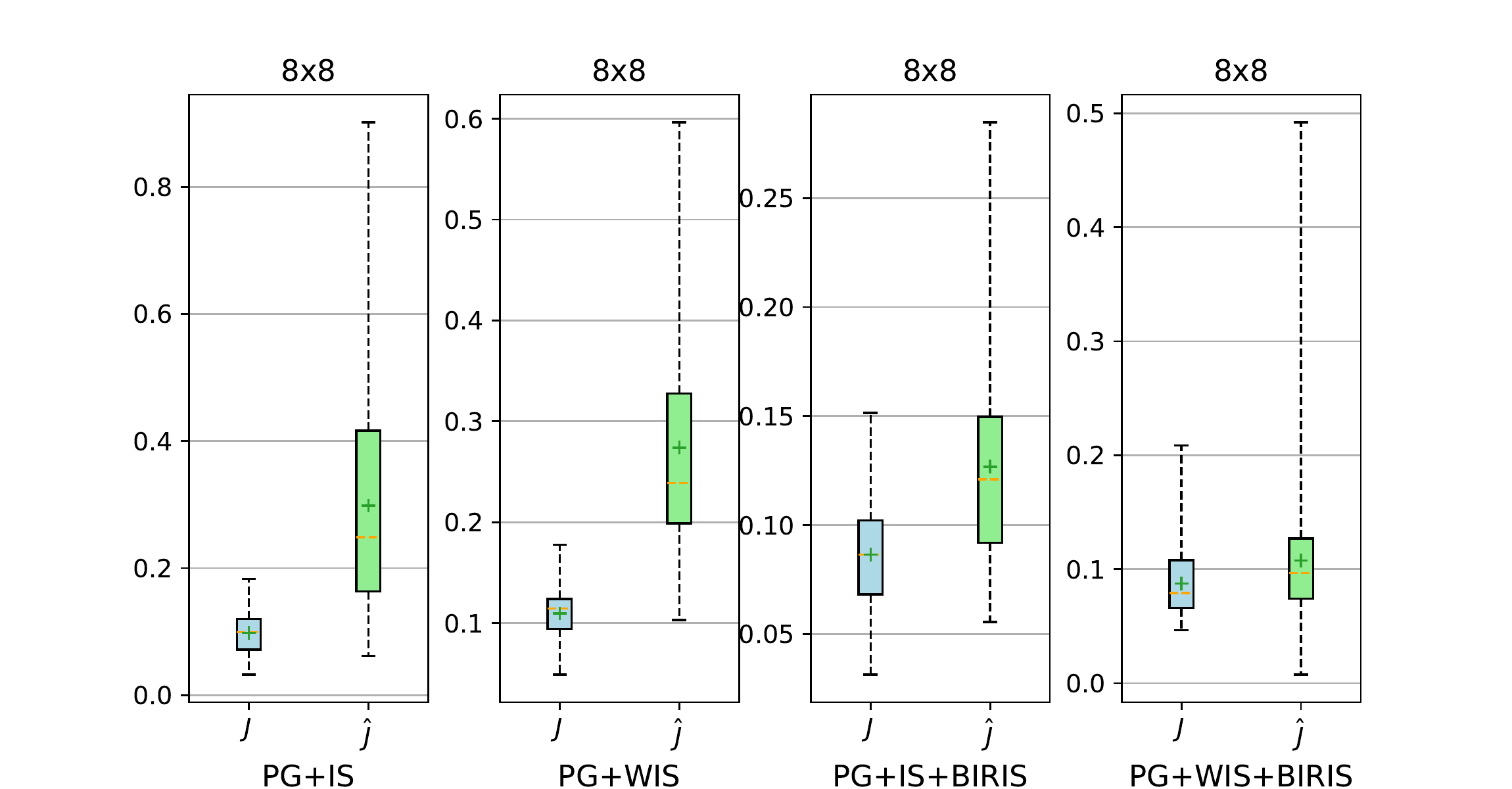}
\label{1b_fig}
\end{minipage}}
\caption{(a) A high-level illustration of the Reuse Bias and our BIRIS. The Reuse Bias is caused by the fact that off-policy methods optimize and evaluate the policy with the same data in the replay buffer. 
(b) Experimental results of PG+IS, PG+WIS, PG+BIRIS+IS, and PG+BIRIS+WIS in MiniGrid 8$\times$8 with replay buffer size 30. In each subfigure, $J$ and $\hat{J}$ represent the expected return and the estimated return via the replay buffer of the target policy respectively. We repeat the experiment 50 times and plot the box diagram, where the orange dashed line represents the mean. This figure shows that Reuse Bias is severe to cause an erroneous policy evaluation and our BIRIS can significantly reduce the Reuse Bias. (More details are in Sec.~\ref{expe_reuse} and Appendix~\ref{appendix-minigrid})
}
\label{minigrid_fig}
\end{figure*}
Previous off-policy methods, which are usually combined with importance sampling, have suffered from the problem of unstable training~\cite{li2015toward}. 
To address this problem, previous work mainly focuses on the high variance of importance sampling~\cite{li2015toward,liu2018breaking,xie2019towards,fujimoto2021deep}. However, we find that reusing trajectories in the replay buffer to optimize and evaluate may introduce a bias to off-policy evaluation (See Fig.~\ref{highlight_fig}), which  may also lead to unstable training and can not be ignored. 

In this work, we present a new analysis beyond variance by introducing the concept of \emph{Reuse Bias}, which formally characterizes the bias in off-policy evaluation caused by the reuse of the replay buffer for evaluation and optimization (See Definition~\ref{df_reuse_bias}). 
Under some mild assumptions (See Theorem~\ref{reused_theorem}, \ref{thm_one_pg}), we theoretically show that off-policy evaluation is an overestimation of the true expected cumulative return. 
To explain how severe the Reuse Bias is, we first prove that the expectation of off-policy evaluation for the optimal policy over the replay buffer may be arbitrarily large when its true return is arbitrarily small (See Theorem~\ref{optimal_theorem}).
Moreover, we experiment on MiniGrid to show that Reuse Bias will cause an erroneous policy evaluation and even gets more significant as the environments become more complex (See Fig.~\ref{1b_fig}, and more details are in Sec.~\ref{expe_reuse}). 
Consequently, the Reuse Bias can be significantly severe and the inaccurate off-policy evaluation may mislead the direction of the gradient. This becomes particularly problematic for further policy improvement because sub-optimal actions might be reinforced in the next policy update, and further induces unstable training. 

To reveal how to optimize the policy for controlling the Reuse Bias, we present a high-probability upper bound of Reuse Bias for any possible policies by extending PAC Bayes techniques~\cite{McAllester98,McAllester03} (See Theorem~\ref{main_theorem}). 
Compared with classical bounds in statistical learning theory, our result holds for any policy hypothesis and is related to the optimized policy for guiding further optimization to control the Reuse Bias.
Moreover, we present that controlling one term, which only depends on the replay buffer and the policy, of our upper bound can effectively control the Reuse Bias by introducing the concept of stability for off-policy algorithms (See Definition~\ref{df_sta}, Theorem~\ref{thm-stability} and~\ref{thm_calculate_beta}).
Based on these analyses, we propose a novel Bias-Regularized Importance Sampling (BIRIS) framework which can improve the sample efficiency of off-policy algorithms by controlling Reuse Bias. We then experiment on MiniGrid to show that our BIRIS can conspicuously reduce the Reuse Bias (See Fig.~\ref{1b_fig}, and more details are in Sec.~\ref{expe_reuse}). Since our method is orthogonal to existing off-policy algorithms and variance-reduction methods, like weighted importance sampling (WIS)~\cite{mahmood2014weighted}, BIRIS can be easily combined with them. Extensive experimental results show that our BIRIS can significantly reduce the Reuse Bias in MiniGrid compared with PG+IS and PG+WIS. Since practical off-policy algorithms for handling complicated continuous control tasks always consider using actor-critic methods rather than directly using trajectories, we also extend our Reuse Bias to AC-based methods and show that these methods also suffer from the negative impact of the Reuse Bias. Moreover, we extend our BIRIS to this more practical setting and provide two practical algorithms (SAC+BIRIS and TD3+BIRIS). Extensive experiments show that our BIRIS-based methods significantly improve the performance and sample efficiency compared with baselines on a series of continuous control tasks in MuJoCo. In summary, our contributions are:
\begin{itemize}
    \item To the best of our knowledge, this is the first attempt to discuss the bias of off-policy evaluation due to reusing the replay buffer. We show that the off-policy evaluation via importance sampling is an overestimation when optimized by the same replay buffer, which is recognized as the Reuse Bias in this paper.
    \item Moreover, we derive a high-probability bound of the Reuse Bias, which indicates that it cannot be eliminated only by increasing the number of trajectory samples. Then we introduce the concept of stability and provide an upper bound for the Reuse Bias via stability, which indicates that controlling one term of the high-probability bound can effectively control the Reuse Bias. 
    \item Based on these analyses, we propose BIRIS and show that our BIRIS can conspicuously reduce the Reuse Bias through experiments in MiniGrid. Moreover, we experimentally show that our method can significantly improve the performance and sample efficiency for different off-policy methods in different MuJoCo tasks. 
\end{itemize}

\section{Related Work}

\paragraph{Off-policy Evaluation.} Designing safe, stable, and effective methods are significant for reinforcement learning (RL)~\cite{garcia2015comprehensive,ying2022towards}, especially for off-policy methods, which hopes to reuse historical trajectories for improving the sample efficiency~\cite{precup2000eligibility,precup2001off}.
It is well known that, when the current policy is independent of the samples, importance sampling is unbiased but has a high variance that is exponentially related to trajectory length~\cite{li2015toward}, which may cause the result unstable. Consequently, a lot of research focuses on reducing the variance in off-policy evaluation, like weighted importance sampling~\cite{mahmood2014weighted}, doubly robust estimator methods~\cite{jiang2016doubly,thomas2016data,su2020doubly}, marginalized importance sampling methods~\cite{liu2018breaking,xie2019towards,fujimoto2021deep}, and so on~\cite{metelli2018policy,hanna2019importance,kumar2020discor,liu2021regret,zhang2021convergence,metelli2021subgaussian}. Some researchers also pay attention to providing a high confidence estimation for off-policy evaluation.~\cite{thomas2015high} first applied concentration inequality to provide a lower bound of off-policy evaluation. Follow-up work considered how to provide confidence intervals~\cite{hanna2017bootstrapping,shi2021deeply} as well as a high-probability lower bound for self-normalized importance weighting~\cite{kuzborskij2021confident}. Some recent work also provided a high-confidence estimation of the variance~\cite{chandak2021high} and other parameters of the return distribution~\cite{metelli2018policy,metelli2021subgaussian,chandak2021universal}. However, little work has been done on systematically examining the bias of off-policy evaluation.

\paragraph{Bias in Off-policy and Offline RL.}
Since bias is significant to analyze and develop algorithms in RL, there are some widely discussed biases in off-policy and offline settings like Q-function overestimation in DQN~\cite{van2016deep}, extrapolation error in offline RL~\cite{fujimoto2019off}, and sampling error in TD learning~\cite{pavse2020reducing}. These bias is mainly caused by out-of-distribution data~\cite{fujimoto2019off} and about Q function estimation in actor-critic methods~\cite{van2016deep,fujimoto2019off}. Besides them, in this work, we first discuss the bias resulting from reusing the trajectories in the replay buffer, which can regard as the bias in the policy optimization in actor-critic methods (more details are in Sec.~\ref{sec_ac}). 


\section{Preliminary}
\label{sec_reuse_bias}
We consider the Markov Decision Process (MDP) of $\mathcal{M} = (\mathcal{S}, \mathcal{A}, \mathcal{P}, \mathcal{R}, \gamma)$ in which $\mathcal{S}$ and $\mathcal{A}$ represent the state and action spaces, respectively. For any state-action pair $(s,a)\in\mathcal{S}\times \mathcal{A}$, $\mathcal{P}(\cdot|s,a)$ is a distribution over $\mathcal{S}$, representing its transition dynamic. Moreover, $\mathcal{R}(s, a)$ is the reward function and  $\gamma\in(0,1)$ is the discount factor. The policy of any agent is modeled as a mapping $\pi:\mathcal{S}\rightarrow \mathcal{D}(\mathcal{A})$, i.e., $\pi(\cdot|s)$ is a distribution over $\mathcal{A}$ given state $s\in\mathcal{S}$. 

In RL, the agent interacts with the environment by making its decision at every timestep. At the beginning, we assume that the agent is in the state $s_0\sim\mu(\cdot)$, where $\mu(\cdot)$ is the initial state distribution. At timestep $t$, the agent chooses its action $a_t\sim\pi(\cdot|s_t)$, arrives at the next state $s_{t+1}\sim\mathcal{P}(\cdot|s_t, a_t)$, and receive a corresponding reward $r_t = \mathcal{R}(s_t,a_t)$. The performance of policy $\pi$ is defined as the expectation of the discounted return of the trajectory:
\begin{equation}
\begin{split}
    &J(\pi) = \mathbb{E}_{\tau\sim\pi}\left[R(\tau)\triangleq\sum_{t=0}^{\infty}\gamma^t r_t\right],
\end{split}
\end{equation}
where $R(\tau)$ represents the discounted return of the trajectory $\tau\triangleq (s_0, a_0, r_0, s_1, a_1, r_1, ...)\in\Gamma$.


As the trajectory distribution $\tau\sim\pi$ is unknown, in practice, we estimate $J(\pi)$ by Monte Carlo sampling. In order to improve sample efficiency, off-policy algorithms utilize trajectories sampled from other policies via importance sampling. 
We denote that there are $m$ trajectories in the replay buffer $\mathcal{B}$ and the trajectory $\tau_i$ is sampled from policy $\hat{\pi}_i$, i.e. $\mathcal{B} = \{\tau_i\}_{i=1}^m\in\Gamma^m, \tau_i\sim\hat{\pi}_i$.

For notational clarity, we denote $\hat{\Pi} = (\hat{\pi}_1, \hat{\pi}_2, ..., \hat{\pi}_m)$ and $\mathcal{B}\sim\hat{\Pi} = \tau_1\sim\hat{\pi}_1, \tau_2\sim\hat{\pi}_2,...,\tau_m\sim\hat{\pi}_m$. In off-policy evaluation, to estimate $J(\pi)$ of a different policy $\pi$, we need to modify the weight of the return for each trajectory in $\mathcal{B}$ via importance sampling~\cite{kahn1953methods,precup2000eligibility}. The resulting importance sampling estimate of the return for $\pi$ is:
\begin{equation}\label{eq:IS-est}
\begin{split}
    &\hat{J}_{\hat{\Pi}, \mathcal{B}}(\pi) = \frac{1}{m}\sum_{i=1}^m \frac{p^{\pi}(\tau_i)}{p^{\hat{\pi}_i}(\tau_i)}  R(\tau_i),
\end{split}
\end{equation}
where $p^{\pi}(\tau)$ represents the probability of generating the trajectory $\tau$ with policy $\pi$, i.e.,
\begin{equation}
    p^{\pi}(\tau) = \mu(s_0)\prod_{i=0}^{\infty}\left[\pi(a_i|s_i)\mathcal{P}(s_{i+1}|s_i,a_i)\right].
\end{equation}

\section{Reuse Bias}
In this section, we first introduce a novel concept of {\it Reuse Bias} to measure the bias in current off-policy evaluation methods. Then, we show that the reuse of the replay buffer causes overestimation and provide a high probability upper bound of the Reuse Bias with some insights into this result. 

\subsection{Reuse Bias}
Assuming that the hypothesis set of policies is $\mathcal{H}$, for any off-policy algorithm, we formalize it as a mapping $\mathcal{O}:\mathcal{H}\times\Gamma^m\rightarrow \mathcal{H}$ that takes an initialized policy $\pi_0\in\mathcal{H}$ and a replay buffer $\mathcal{B}\in\Gamma^m$ as inputs, and outputs a policy $\mathcal{O}(\pi_0, \mathcal{B})\in\mathcal{H}$.
In off-policy evaluation, we use the importance sampling estimate $\hat{J}_{\hat{\Pi}, \mathcal{B}}(\mathcal{O}(\pi_0, \mathcal{B}))$ defined in Eq.~(\ref{eq:IS-est}) as an approximation of the expected return $J(\mathcal{O}(\pi_0, \mathcal{B}))$. 
We define their difference as the Reuse Error of an off-policy algorithm $\mathcal{O}$ on $\pi_0$ as well as $\mathcal{B}$, and the expectation of the Reuse Error as the Reuse Bias:
\begin{definition}[\textbf{Reuse Bias}]
\label{df_reuse_bias}
For any off-policy algorithm $\mathcal{O}$, initialized policy $\pi_0$ and replay buffer $\mathcal{B}\sim\hat{\Pi}$, we define the Reuse Error of $\mathcal{O}$ on $\pi_0$ and $\mathcal{B}$ as
\begin{equation}
    \epsilon_{\mathrm{RE}}(\mathcal{O}, \pi_0, \mathcal{B}) \triangleq \hat{J}_{\hat{\Pi}, \mathcal{B}}(\mathcal{O}(\pi_0, \mathcal{B})) - J(\mathcal{O}(\pi_0, \mathcal{B})).
\end{equation}
Moreover, we define is expectation as the Reuse Bias:
\begin{equation}
    \epsilon_{\mathrm{RB}}(\mathcal{O}, \pi_0) \triangleq \mathbb{E}_{\mathcal{B}}[\epsilon_{\mathrm{RE}}(\mathcal{O}, \pi_0, \mathcal{B})].
\end{equation}
\end{definition}
It is well-known that importance sampling is an unbiased estimation when the estimated distribution is independent of the samples. In other words, the off-policy evaluation is unbiased, i.e., $\epsilon_{\mathrm{RB}}(\mathcal{O}, \pi_0) = \mathbb{E}_{\mathcal{B}} \left[\epsilon_{\mathrm{RE}}(\mathcal{O}, \pi_0, \mathcal{B})\right] = 0$, when $\mathcal{O}(\pi_0, \mathcal{B})$ is independent of $\mathcal{B}$. However, off-policy algorithms~\cite{sac,td3} utilize trajectories in the replay buffer $\mathcal{B}$ to optimize the policy. In this case, the policy $\mathcal{O}(\pi_0, \mathcal{B})$ does depend on $\mathcal{B}$ which makes the off-policy evaluation no longer unbiased. And we will provide a first systematic investigation of the Reuse Bias, as detailed below. 


We first assume that the off-policy algorithm $\mathcal{O}^*$ satisfies the condition that $\mathcal{O}^*(\pi_0, \mathcal{B})$ has the highest estimated performance with $\mathcal{B}$ and consider the Reuse Bias under this assumption. This assumption is mild and practical since off-policy algorithms always hope to maximize the expected performance of the policy w.r.t. the replay buffer in practice.
Under this assumption, the estimation $\hat{J}_{\hat{\Pi},\mathcal{B} }(\mathcal{O}^*(\pi_0, \mathcal{B}))$ is an overestimation of $J(\mathcal{O}^*(\pi_0, \mathcal{B}))$ as summarized below:
\begin{theorem}[Overestimation for Off-Policy Evaluation, Proof in Appendix~\ref{reused_proof}]
\label{reused_theorem}
Assume that $\mathcal{O}^*(\pi_0, \mathcal{B})$ is the optimal policy of $\mathcal{H}$ over the replay buffer $\mathcal{B}$, i.e.,
\begin{equation}
\begin{split}
    \mathcal{O}^*(\pi_0, \mathcal{B}) &= \mathop{\arg\max}_{\pi\in\mathcal{H}} \hat{J}_{\hat{\Pi}, \mathcal{B}}(\pi)\\
    &= \mathop{\arg\max}_{\pi\in\mathcal{H}} \frac{1}{m} \sum_{i=1}^m\left[\frac{p^{\pi}(\tau_i)}{p^{\hat{\pi}_i}(\tau_i)}  R(\tau_i)\right].
\end{split}
\end{equation}
We can show that $\hat{J}_{\hat{\Pi}, \mathcal{B}}(\mathcal{O}^*(\pi_0, \mathcal{B}))$ is an overestimation of $J(\mathcal{O}^*(\pi_0, \mathcal{B}))$, i.e., $\epsilon_{\mathrm{RB}}(\mathcal{O}^*, \pi_0) = \mathbb{E}_{\mathcal{B}\sim\hat{\Pi}}\left[\epsilon_{\mathrm{RE}}(\mathcal{O}^*, \pi_0, \mathcal{B})\right]$ $\ge 0$. 
If the equality holds, then for any $\mathcal{B}, \mathcal{B}'\sim\hat{\Pi}$, we have
\begin{equation}
\begin{split}
    \mathcal{O}^*(\pi_0, \mathcal{B}) &= \mathop{\arg\max}_{\pi\in\mathcal{H}} \hat{J}_{\hat{\Pi}, \mathcal{B}'}(\pi).
\end{split}
\end{equation}
\end{theorem}
However, in practice, we always optimize our policy through several steps of gradient descent rather than directly getting the optimal policy $\mathcal{O}^*(\pi_0, \mathcal{B})$ over the replay buffer $\mathcal{B}$. To analyze the Reuse Bias in more practical cases, we consider the one-step policy gradient optimization and prove that the estimation $\hat{J}_{\hat{\Pi},\mathcal{B} }$ is strictly an overestimation of $J$ in this case only under some mild assumptions as:
\begin{theorem}[Overestimation for One-Step PG, Proof in Appendix~\ref{proof_one_pg}]
\label{thm_one_pg}
Given a parameterized policy $\pi_{\theta}$ which is independent with the replay buffer $\mathcal{B}$ and is differentiable to the parameter $\theta$, we consider the one-step policy gradient
\begin{equation}
    \theta' = \theta + \alpha \nabla_{\theta} \hat{J}_{\hat{\Pi}, \mathcal{B}}(\pi_{\theta}),
\end{equation}
where $\alpha$ is the learning rate. If $\nabla_{\theta} \hat{J}_{\hat{\Pi}, \mathcal{B}}(\pi_{\theta})$, as the function of $\mathcal{B}$, is \textbf{not} constant, and $\alpha > 0$ is sufficiently small, then the Reuse Bias is strictly larger than 0, i.e.,
\begin{equation}
    \mathbb{E}_{\mathcal{B}\sim\hat{\Pi}}\hat{J}_{\hat{\Pi}, \mathcal{B}}(\pi_{\theta'}) > \mathbb{E}_{\mathcal{B}\sim\hat{\Pi}}J(\pi_{\theta'}).
\end{equation}
\end{theorem}
Furthermore, to explain how severe the Reuse Bias will be, we provide Theorem~\ref{optimal_theorem} to show that the estimated return of $\mathcal{O}^*(\pi_0, \mathcal{B})$ might be arbitrarily large when its true return is arbitrarily small, which indicates that the Reuse Bias will result in the degeneration of the policy optimization and the final performance.
\begin{theorem}[Proof in Appendix~\ref{proof_optimal_theorem}]
\label{optimal_theorem}
For any fixed replay buffer size $n$, assume that $\Pi$ is the set of random policies, and the reward of any state-action pair $(s,a)$ is not negative, i.e., $\mathcal{R}(s,a) \ge 0$, then there exists an environment satisfying that $ \mathbb{E}_{\mathcal{B}\sim\hat{\Pi}}\left[J(\mathcal{O}^*(\pi_0, \mathcal{B}))\right]$ is arbitrarily small when $\mathbb{E}_{\mathcal{B}\sim\hat{\Pi}}\left[\hat{J}_{\hat{\Pi}, \mathcal{B}}(\mathcal{O}^*(\pi_0, \mathcal{B}))\right]$ is arbitrarily large. In other words, for $\forall n,M\ge 1, \epsilon > 0, \exists$ an environment satisfying that
\begin{equation}
\begin{split}
    &J(\pi^*)=1, |\mathcal{B}| = n, \mathbb{E}_{\mathcal{B}\sim\hat{\Pi}}\left[J(\mathcal{O}^*(\pi_0, \mathcal{B}))\right] \leq \epsilon,\\
    &\mathbb{E}_{\mathcal{B}\sim\hat{\Pi}}\left[\hat{J}_{\hat{\Pi}, \mathcal{B}}(\mathcal{O}^*(\pi_0, \mathcal{B}))\right] \ge M.
\end{split}
\end{equation}
here $\pi^*$ is the optimal policy of this environment.
\end{theorem}
In Theorem~\ref{optimal_theorem}, we point out that the Reuse Bias may cause an erroneous off-policy policy evaluation, which might be sufficiently severe, and affect further policy improvement, which is also shown in experiments of MiniGrid (See Sec.~\ref{expe_reuse}).



\subsection{High Probability Bound for Reuse Error}
Since the Reuse Bias may mislead the optimization direction and further affect the performance, which cannot be ignored, it is important to analyze the connection of the replay buffer $\mathcal{B}$ and the Reuse Error to better control it. Without losing generality, we mainly consider the case that all trajectories are sampled from the same original policy $\hat{\pi}$. 

When the hypothesis set has bounded statistical complexity, for example, $\mathcal{H}$ is finite or its VC-dimension is finite, we can naturally provide a high-probability upper bound of the Reuse Error via using some classical results of statistical learning theory~\cite{unserstandingml} (We provide the result when $\mathcal{H}$ is finite in Appendix~\ref{appendix_bounded_hypo} as an example). 
However, these results are based on the statistic complexity of the hypothesis set and $\max_{\pi, \tau}\left[\frac{p^{\pi}(\tau)}{p^{\hat{\pi}}(\tau)}\right]$. In practice, the statistic complexity of the hypothesis set is extremely huge while the sample number $m$ is relatively small, and $\max_{\pi, \tau}\left[\frac{p^{\pi}(\tau)}{p^{\hat{\pi}}(\tau)}\right]$ is the term related on the environment as well as the hypothesis set, which can not be optimized in the training stage.

Here, we provide another high-probability upper bound for the Reuse Error as below. Compared with prior results, our bound holds for all $\mathcal{H}$ even with unbounded statistical complexity, and is related with the optimized policy $\mathcal{O}(\pi_0, \mathcal{B})$, which guides us to optimize $\mathcal{O}(\pi_0, \mathcal{B})$ by controlling Reuse Error.
\begin{theorem}[High-Probability Bound for Reuse Error, Proof in Appendix~\ref{proof_main_theorem}]
\label{main_theorem}
Assume that, for any trajectory $\tau$, we can bound its return as $0\leq R(\tau)\leq 1$. Then, for \textbf{any} off-policy algorithm $\mathcal{O}$ and initialized policy $\pi_0\in\mathcal{H}$, with a probability of at least $1-\delta$ over the choice of an i.i.d. training set $\mathcal{B} = \{\tau_i\}_{i=1}^m$ sampled by the same original policy $\hat{\pi}$, the following inequality holds:
\begin{equation}
    |\epsilon_{\mathrm{RE}}(\mathcal{O}, \pi_0, \mathcal{B})| \leq \sqrt{ \frac{m \epsilon_1 +\log\left(\frac{m^2}{\delta}\right)}{m-1}} + \epsilon_2,
\end{equation}
where $\epsilon_1$ and $\epsilon_2$ are defined as:
\begin{equation}
\begin{split}
    \epsilon_1 &= \mathrm{KL}[p^{\mathcal{O}(\pi_0, \mathcal{B})}(\cdot)||p^{\hat \pi}(\cdot)],\\
    \epsilon_2 &= \frac{1}{m}\sum_{i=1}^m \left|1 - \frac{p^{\mathcal{O}(\pi_0, \mathcal{B})}(\tau_i)}{p^{\hat\pi}(\tau_i)} \right|
    = \mathbb{E}_{\tau\sim\mathcal{B}} \left|1 - \frac{p^{\mathcal{O}(\pi_0, \mathcal{B})}(\tau)}{p^{\hat\pi}(\tau)} \right|.
\end{split}
\end{equation}
\end{theorem}
In Theorem~\ref{main_theorem}, $\epsilon_1$, as the KL divergence of the trajectory distribution under $\mathcal{O}(\pi_0, \mathcal{B})$ and $\hat{\pi}$, indicates their similarity over the trajectory space $\Gamma$. Moreover, $\epsilon_2$ only focuses on the similarity in the replay buffer $\mathcal{B}$ of these two policies. Then we will introduce some insights into this theorem.

First, when the replay buffer size $m$ tends to infinity, our high-probability bound will converge to $\sqrt{\epsilon_1} + \epsilon_2$ rather than 0. This is reasonable because the policy $\mathcal{O}(\pi_0, \mathcal{B})$ depends on $\mathcal{B}$ and the hypothesis space in our results may be arbitrary large since the statistical complexity of practical policies might be very huge or even difficult to analyze.
Also, for better explaining that the Reuse Bias may still exist even when the sample number $m$ tends to infinity, we conduct a concrete example satisfying that $|\epsilon_{\mathrm{RE}}(\mathcal{O},\hat{\pi},\mathcal{B})| = 1$ holds for any finite $m$ and any replay buffer $\mathcal{B}$ in Appendix~\ref{appendix_example}. 

Moreover, the high-probability upper bound depends on the KL divergence of these two policies as well as on their probability over $\mathcal{B}$. Consequently, controlling their divergence, especially over $\mathcal{B}$, guides us on how to optimize the policy for controlling the Reuse Bias.

\section{Methodology}
In this section, introducing the concept of stability for off-policy algorithms, we further analyze the connection between controlling $\epsilon_2$ in Theorem~\ref{main_theorem} and controlling the Reuse Bias. Moreover, we propose a general framework for controlling Reuse Bias with practical algorithms by extending actor-critic methods for complicated continuous control tasks. 

\subsection{Theoretical Analysis on Controlling Reuse Bias}
Inaccurate off-policy evaluation caused by Reuse Bias may lead the policy optimization to the wrong direction and further affect both the performance and sample utilization. Therefore, it is necessary to develop methods to control the Reuse Bias. 

Based on Theorem~\ref{main_theorem}, we can control $\epsilon_1$ and $\epsilon_2$ to control the Reuse Bias. However, directly using the trajectories in the replay buffer to estimate $\epsilon_1$, which is the KL divergence of $p^{\mathcal{O}(\pi_0, \mathcal{B})}$ and $p^{\hat{\pi}}$, may also be biased because $\mathcal{O}(\pi_0, \mathcal{B})$ is dependent on the replay buffer. 
Thus we only consider constraining $\epsilon_2$ and use $\mathcal{L}(\pi, \mathcal{B})$ to denote it, i.e.,
\begin{equation}
\begin{split}
    \mathcal{L}(\pi, \mathcal{B}) &\triangleq \mathbb{E}_{\tau\sim\mathcal{B}} \left|1 - \frac{p^{\pi}(\tau)}{p^{\hat\pi}(\tau)} \right|.
\end{split}
\end{equation}
Moreover, we hope to analyze the Reuse Bias of policies with bounded $\mathcal{L}(\pi, \mathcal{B})$. Borrowing ideas of stability~\cite{hardt2016train}, we similarly define the concept of stability for off-policy algorithms and further provide an upper bound of stable off-policy algorithms as shown below
\begin{definition}[Stability for Off-Policy Algorithm]
\label{df_sta}
A randomized off-policy algorithm $\mathcal{O}$ is $\beta$-uniformly stable if for all Replay Buffer $\mathcal{B}, \mathcal{B}'$, such that $\mathcal{B}, \mathcal{B}'$ differ in at most one trajectory, we have
\begin{equation}
    \forall \tau, \pi_0,\quad \mathbb{E}_{\mathcal{O}}\left[p^{\mathcal{O}(\pi_0, \mathcal{B})}(\tau) - p^{\mathcal{O}(\pi_0, \mathcal{B}')}(\tau)\right]\leq\beta.
\end{equation}
\end{definition}
\begin{theorem}[Bound for the Reuse Error of Stable Algorithm, Proof in Appendix~\ref{proof-stability}]
\label{thm-stability}
Suppose a randomized off-policy algorithm $\mathcal{O}$ is $\beta$-uniformly stable, then we can prove that
\begin{equation}
    \forall \pi_0,\quad \left|\mathbb{E}_{\mathcal{B}\sim\hat{\pi}}\mathbb{E}_{\mathcal{O}}\left[\epsilon_{\mathrm{RE}}(\mathcal{O}, \pi_0, \mathcal{B})\right]\right| \leq \beta.
\end{equation}
\end{theorem}
Theorem~\ref{thm-stability} controls Reuse Bias for stable off-policy algorithms Furthermore, we provide Theorem~\ref{thm_calculate_beta} to point out that, under some mild assumptions, Reuse Bias of off-policy stochastic policy gradient can be controlled just by $\mathcal{L}(\pi, \mathcal{B})$.
\begin{theorem}[Details and Proof are in Appendix~\ref{proof_calculate_beta}]
\label{thm_calculate_beta}
We assume that the policy $\pi_{\theta}$ is parameterized with $\theta$, and $\left|\nabla_{\theta} \log p^{\pi_{\theta}}(\tau)\right|\leq L_1$ holds for any $\theta, \tau$, and $p^{\pi_{\theta}}(\tau)$ is $L_2$-Lipsticz to $\theta$ for any $\tau$. If we constrain the policy by $\mathcal{L}(\pi, \mathcal{B}) \leq M$, then off-policy stochastic policy gradient algorithm (detailed in Appendix ~\ref{proof_calculate_beta})
is $\beta$-uniformly stable where $\beta$ is positively correlated with $M$, $L_1$ and $L_2$.
\end{theorem}

\begin{table*}[t]
\centering
\footnotesize
\begin{tabular}{cccccccr}
\toprule
Size of Replay Buffer&Method& 5$\times$5 &  5$\times$5-random& 6$\times$6& 6$\times$6-random& 8$\times$8 &16$\times$16\\
\hline
\multirow{4}*{30}&PG+IS
& 0.57 & 0.26 & 0.86 & 0.55 & 2.04 & 19.04 \\
&PG+WIS
& 0.36 & 0.24 & 0.72 & 0.43 & 1.50 & 5.75 \\
&PG+IS+BIRIS
& 0.19 & \textbf{0.08} & \textbf{0.11}& \textbf{0.25} & 0.47 & 0.43 \\
&PG+WIS+BIRIS
& \textbf{0.16} & 0.12 & 0.17 & 0.28 & \textbf{0.23} & \textbf{0.05} \\
\hline
\multirow{4}*{40}&PG+IS
& 0.38 & 0.24 & 0.67 & 0.29 & 1.99 & 80.62 \\
&PG+WIS
& \textbf{0.20} & 0.21 & 0.49 & 0.32 & 1.20 & 4.75 \\
&PG+IS+BIRIS
& 0.23 & 0.18 & 0.29 & 0.25 & 0.39 & \textbf{0.44} \\
&PG+WIS+BIRIS
& 0.21 & \textbf{0.14} & \textbf{0.25}& \textbf{0.24} & \textbf{0.26} & 0.51 \\
\hline
\multirow{4}*{50}&PG+IS
& 0.44 & 0.21 & 0.60 & 0.42 & 2.01 & 13.40 \\
&PG+WIS
& 0.26 & 0.22 & 0.51 & 0.31 & 1.22 & 4.65 \\
&PG+IS+BIRIS
& 0.23 & 0.20 & \textbf{0.17} & \textbf{0.11} & \textbf{0.23} & 0.26 \\
&PG+WIS+BIRIS
& \textbf{0.14} & \textbf{0.18} & 0.25 & 0.16 & 0.24 & \textbf{0.21} \\
\bottomrule
\end{tabular}
\caption{Relative Reuse Bias of policies trained by PG+IS, PG+WIS, PG+IS+BIRIS, and PG+WIS+BIRIS in different MiniGrid environments.}
\label{biris_table}
\end{table*}

\subsection{Bias-Regularized Importance Sampling Framework}
Based on above analyses, for any given off-policy algorithm $\mathcal{O}$, we focus on controlling $\mathcal{L}(\mathcal{O}(\pi_0, \mathcal{B}), \mathcal{B})$. To simplify the notation, we denote $\pi\triangleq\mathcal{O}(\pi_0, \mathcal{B}) $ as the policy trained by the off-policy algorithm. For any trajectory $\tau$, we have
\begin{equation}
\begin{split}
    \frac{p^{\pi}(\tau)}{p^{\hat{\pi}}(\tau)}
    =& \frac{\mu(s_0)}{\mu(s_0)} \prod_{i} \left[ \frac{\pi(a_i|s_i)\mathcal{P}(s_{i+1}|s_i, a_i)}{\hat{\pi}(a_i|s_i)\mathcal{P}(s_{i+1}|s_i, a_i)}\right]
    = \prod_{i} \frac{\pi(a_i|s_i)}{\hat{\pi}(a_i|s_i)}.
\end{split}
\end{equation}
Thus, we can derive $\mathcal{L}(\pi, \mathcal{B})$ as
\begin{equation}
\begin{split}
    \mathcal{L}(\pi, \mathcal{B}) &= \mathbb{E}_{\tau\sim\mathcal{B}} \left|1 - \frac{p^{\pi}(\tau)}{p^{\hat\pi}(\tau)} \right|
    = \mathbb{E}_{\tau\sim\mathcal{B}} \left|1 - \prod_{i} \frac{\pi(a_i|s_i)}{\hat{\pi}(a_i|s_i)} \right|.
\end{split}
\end{equation}
Consequently, we propose a novel off-policy reinforcement learning framework named Bias-Regularized Importance Sampling (BIRIS) for controlling the Reuse Bias as below
\begin{equation}
\begin{split}
    &\mathcal{O}_{\text{BIRIS}}(\pi_0, \mathcal{B}) = \mathop{\arg\min}_{\pi\in\mathcal{H}} \mathcal{L}_{\text{BIRIS}}(\pi,\mathcal{B}),\\
    \text{where } &\mathcal{L}_{\text{BIRIS}}(\pi,\mathcal{B}) = \mathcal{L}_{\text{RL}}(\pi, \mathcal{B}) + \alpha\mathcal{L}(\pi, \mathcal{B}),
\end{split}
\end{equation}
here $\mathcal{L}_{\text{RL}}$ denotes the nominal loss function of the basic off-policy algorithm, and $\alpha\ge 0$ is a hyperparameter to control the trade-off between standard RL loss and regularized loss. 

\subsection{Connection with Actor-Critic Methods}
\label{sec_ac}
However, since directly using the return of trajectories to optimize the policy suffers from instability, many off-policy practical methods for complicated environments consider using TD-based actor-critic methods, like Soft Actor Critic (SAC)~\cite{sac} and Twin Delayed Deep Deterministic Policy Gradient (TD3)~\cite{td3}. These methods always alternately optimize the critic and the actor as below
\begin{itemize}
    \item Optimizing the Q function via TD-learning, i.e., minimizing $\sum_{(s,a)\in\mathcal{B}} (Q_{\theta}(s,a) - Q_\text{target}(s,a))^2$
    \item Optimize the policy via the Q function, i.e., maximizing $\sum_{s\in\mathcal{B}} Q_{\theta}(s,\pi_{\phi}(s))$
\end{itemize}
Current work mainly considers the error in Q function optimization caused by out-of-distribution data. For example, Extrapolation Error~\cite{fujimoto2019off} is because erroneously estimating target Q values of unseen state-action pairs.

However, similar to our analyses above, in the policy optimization stage, reusing the state-action pairs in the replay buffer will also cause the Reuse Error, and we can also implement our BIRIS to reduce the negative impact. Moreover, for complicated environments, trajectories are so long that their probabilities are difficult to calculate and numerically unstable. If we assume $\left|\frac{\pi(a|s)}{\hat{\pi}(a|s)}\right|\leq \epsilon$ for all $(s,a)\in\mathcal{B}$ and the length of the trajectory is $T$, we can show that
\begin{equation}
    \left|\prod_{i=1}^T\frac{\pi(a_i|s_i)}{\hat{\pi}(a_i|s_i)} - 1\right|\leq \max(1-(1-\epsilon)^T, (1+\epsilon)^T-1).
\end{equation}
Consequently, for handling complicated environments, we use the state-level ratio $\mathcal{L}_{\text{BR}}$ as a surrogate
\begin{equation}
    \mathcal{L}_{\text{BR}}(\pi, \mathcal{B}) \triangleq \mathbb{E}_{(s,a)\in\mathcal{B}} \left|\frac{\pi(a|s)}{\hat{\pi}(a|s)} - 1\right|.
\end{equation}
Our $\mathcal{L}_{\text{BR}}$ has some similar points with pessimistic offline reinforcement learning methods~\cite{kumar2020conservative,liu2020provably,fujimoto2021minimalist} and is a supplement to the effectiveness of pessimistic methods based on our theoretical analyses on the view of the Reuse Bias. 
Furthermore, for handling complicated continuous control tasks, we concretize our framework to specific algorithms, like SAC and TD3, of which the core is to calculate $\frac{\pi(a|s)}{\hat{\pi}(a|s)}$ for any state-action pair $(s,a)\in\mathcal{B}$ (More details are in Appendix~\ref{appendix_biris}).

\section{Experiments}
We now present empirical results to answer the questions:
\begin{itemize}
    \item How severe is Reuse Bias in the practical experiments and can our BIRIS effectively reduce Reuse Bias?
    \item What is the empirical performance of our BIRIS for actor-critic methods in complicated continuous control tasks?
\end{itemize}

\subsection{Experiment Setup}
We experiment in two different domains to answer the questions above. Before that, we first provide a short description of environments, including algorithms as well as metrics. 

\paragraph{Gridworld.} For the first question, we experiment in MiniGrid\footnote{https://github.com/mit-acl/gym-minigrid},
which includes different shapes of grids with discrete state space and action space, and is simple for optimizing and evaluating our policy.
We will calculate and compare the Reuse Bias of PG+IS, PG+WIS~\cite{mahmood2014weighted}, PG+IS+BIRIS, and PG+WIS+BIRIS.

\paragraph{Simulated Robotics.} To evaluate BIRIS in complicated environments, we choose several continuous control tasks from MuJoCo
and use two popular off-policy algorithms: SAC~\cite{sac} and TD3~\cite{td3}, as baselines, with uniform sampling and prioritization experience replay (PER)~\cite{schaul2016prioritized}. Our implementation is based on Tianshou~\cite{weng2021tianshou}.


\begin{table*}[t]
\centering
\footnotesize
\begin{tabular}{cccccc}
\toprule
Method&Ant&HalfCheetah&Humanoid&Walker2d&InvertedPendulum\\ 
\midrule
SAC
&5797.9$\pm$492.1
&12096.6$\pm$597.7		
&5145.4$\pm$567.4
&4581.1$\pm$541.4
&\textbf{1000.0$\pm$0.0}\\
SAC+PER
&\textbf{6133.5$\pm$269.0}
&11695.1$\pm$603.2
&4860.8$\pm$1117.1
&4320.5$\pm$392.5
&\textbf{1000.0$\pm$0.0}\\
SAC+BIRIS
&5843.8$\pm$159.9
&\textbf{12516.5$\pm$613.3}
&\textbf{5466.1$\pm$493.9}
&\textbf{4836.3$\pm$405.6}
&\textbf{1000.0$\pm$0.0}\\
\hline
TD3
&5215.7$\pm$488.2 
&10147.6$\pm$1291.6
&5012.9$\pm$211.1
&\textbf{4223.0$\pm$350.5}
&\textbf{1000.0$\pm$0.0}\\
TD3+PER
&5351.1$\pm$530.1
&10091.4$\pm$830.3
&4365.5$\pm$608.3
&3879.6$\pm$557.2
&\textbf{1000.0$\pm$0.0}\\
TD3+BIRIS
&\textbf{5675.1$\pm$132.6}
&\textbf{10774.2$\pm$907.0}
&\textbf{5117.9$\pm$181.6}
&4189.4$\pm$485.9 &\textbf{1000.0$\pm$0.0}\\
\bottomrule
\end{tabular}
\caption{Cumulative reward (mean $\pm$ one std) of the best policy trained by SAC, SAC+PER, SAC+BIRIS, TD3, TD3+PER, TD3+BIRIS in different MuJoCo games. In each column, we \textbf{bold} the best performance over all algorithms.}
\label{performance_table}
\end{table*}

\begin{figure*}[t]
\centering
\includegraphics[height=3.5cm,width=15.75cm]{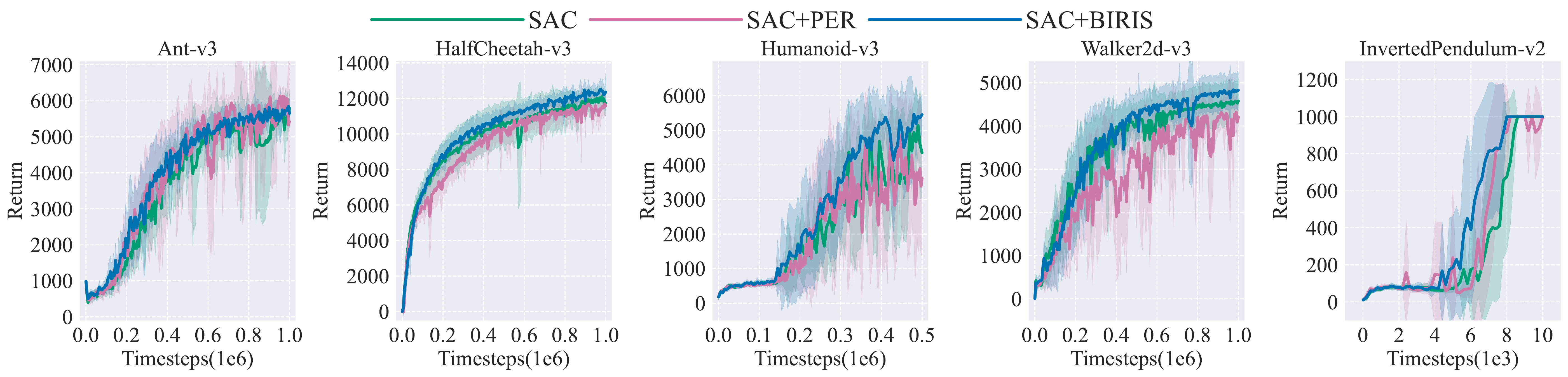}
\includegraphics[height=3.5cm,width=15.75cm]{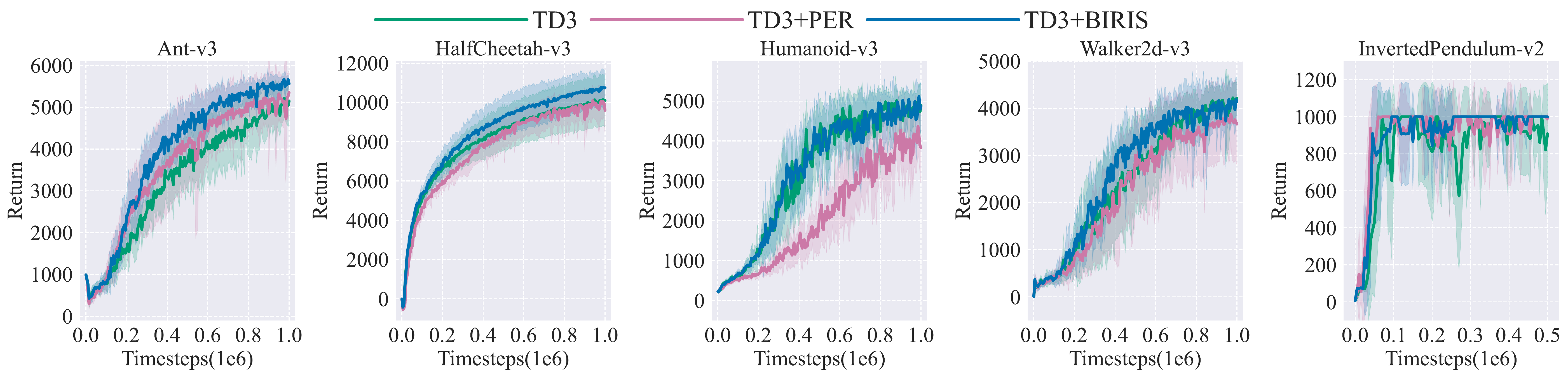}
\caption{Cumulative reward curves for SAC, SAC+PER, SAC+BIRIS, TD3, TD3+PER, and TD3+BIRIS. The x-axes indicate the number of steps interacting with the environment, and the y-axes indicate the performance of the agent, including average rewards with std.}
\label{mujoco_fig}
\end{figure*}



\subsection{Results for Reuse Bias}
\label{expe_reuse}
We first present our experimental results in MiniGrid to answer the first question, i.e., how severe the Reuse Bias is for practical off-policy evaluation and whether our BIRIS can significantly control the Reuse Bias. For different grids ($5\times 5$, $5\times 5$-random, $6\times 6$, $6\times 6$-random, $8\times 8$, $16\times 16$), we initialize our policy parameterized with a simple three-layer convolutional neural network. Moreover, we choose 30, 40, and 50 as the replay buffer size and sample trajectories to fill the replay buffer, to test its impact. Then we use PG~\cite{PG} and PG+BIRIS, respectively, to optimize the policy by maximizing the estimated return via replay buffer by IS and weighted importance sampling (WIS)~\cite{mahmood2014weighted}. In the test stage, we compute the return of our policy estimated by the replay buffer as well as the true expected return of our policy. To mitigate the impact of randomness caused by the replay buffer and sampling, we repeat our experiments for each task and each algorithm with 50 different random seeds.

We report part of our results (8$\times$8 with the size of replay buffer 30) in Fig.~\ref{1b_fig} in the form of box diagrams (All figures are in Appendix~\ref{appendix-minigrid}
). In each subgraph, the distribution of the true return
$J$ of the optimized policy is shown on the left (filled in with blue) while the distribution of the estimated return
$\hat{J}$ is on the right (filled in with green). In each box diagram, the horizontal lines at the top and bottom indicate the largest and smallest value respectively, while the horizontal lines at the top and bottom of the box indicate the first quartile and third quartile respectively. Moreover, the orange dotted line and the green plus sign represent the median and the expectation separately. 
In Table~\ref{biris_table}, we also report the Relative Reuse Bias, i.e., $\epsilon_{\mathrm{RB}}(\mathcal{O}, \pi_0, \mathcal{B})/ J(\mathcal{O}(\pi_0, \mathcal{B}))$ of all algorithms.

As shown in Fig.~\ref{1b_fig} and Table~\ref{biris_table}, the estimated return is obviously higher than the true return in all environments for PG+IS, which matches our theoretical analyses, i.e., using the same replay buffer to optimize and evaluate the policy will lead to overestimation. Furthermore, as shown in Table~\ref{biris_table}, the Relative Reuse Bias shown in this experiment is greater when the grid is larger
and more complicated, which means that the Reuse Bias will be more serious in more complicated environments. It is probably because the policy is more likely to converge to local optimization in complicated environments. 

As shown in Table~\ref{biris_table}, when the replay buffer size increases, the relative Reuse Bias of PG+IS will reduce to some degree but is still large even if the size is 50. Moreover, the relative Reuse Bias of PG+WIS is lower than it of PG+IS but is still relatively large, which shows that variance-reduced methods like WIS can somehow reduce the impact of the Reuse Bias. Finally, in almost all tasks and all replay buffer sizes, the Relative Reuse Bias of PG+IS+BIRIS and PG+WIS+BIRIS are conspicuously smaller than PG+IS and PG+WIS, which means that our BIRIS can significantly reduce the Reuse Bias.

\subsection{Results for Simulated Robotics}

We now report results to evaluate BIRIS-based methods in more complicated environments against actor-critic methods. We compare the performance of our methods against common off-policy algorithms like SAC and TD3 in MuJoCo environments, including Ant, Halfcheetah, Humanoid, Walker2d, and InvertedPendulum. To mitigate the effects of randomness caused by stochastic environments and policies, we train 10 policies with different random seeds for each algorithm in each task and plot the mean with standard deviations of these 10 policies as a function of timestep in the training stage. Moreover, we use the solid line and the part with a lighter color to represent the average reward and standard deviations of 10 strategies, respectively. We also report the optimal mean and standard deviations of the cumulative return of 10 policies trained by each algorithm in each environment in Table~\ref{performance_table}. As shown in this table, PER does not outperform uniform sampling, which is also reported in~\cite{novati2019remember} for DDPG~\cite{ddpg} in MuJoCo.

In Fig.~\ref{mujoco_fig} and Table~\ref{performance_table}, we compare SAC, SAC+PER, and SAC+BIRIS. As we can see, SAC+BIRIS learns a better policy compared with SAC, i.e., SAC+BIRIS can effectively improve the performance and the sample efficiency compared with SAC. Overall, the performance of SAC+BIRIS increases more smoothly in the training stage, especially in HalfCheetah and Humanoid. This is consistent with our theoretical results, in that BIRIS can better estimate off-policy performance and provide a more precise optimization direction.

The comparison of TD3, TD3+PER, and TD3+BIRIS is reported in Fig.~\ref{mujoco_fig} and Table~\ref{performance_table}. Similarly, for all tasks, TD3+BIRIS performs better than TD3, i.e., TD3+BIRIS can improve the performance as well as the sample efficiency compared with TD3, especially in Ant and HalfCheetah. We also notice that TD3+BIRIS can conspicuously reduce the standard deviations of the policies compared with TD3 in Ant and HalfCheetah. This is also because BIRIS can better reduce the Reuse Bias and provide a more stable optimization direction. Moreover, we also do some ablation studies in Appendix~\ref{appendix_supp_expe}.


\section{Conclusion}
In this paper, we first show that the bias in off-policy evaluation is problematic and introduce the concept of Reuse Bias to describe it. 
We theoretically prove the overestimation bias in off-policy evaluation, which is because of reusing historical trajectories in the same replay buffer.
Moreover, we provide a high-probability upper bound of Reuse Bias as well as an expectation upper bound of Reuse Bias for stable off-policy algorithms. Based on these analyses, we provide a framework of Bias-Regularized Importance Sampling (BIRIS) with practical algorithms for controlling Reuse Bias. 
Experimental results demonstrate that BIRIS can significantly reduce Reuse Bias and mitigating Reuse Bias via BIRIS can significantly improve the sample efficiency for off-policy methods in MuJoCo. As our modifications are independent of existing variance reduction methods, they can be easily integrated. 

\newpage
\section*{Ethical Statement}
Off-policy methods are important to improve sample efficiency and apply RL in larger scenarios. This paper studies the Reuse Bias of off-policy evaluation and is beneficial for further studying more efficient and stable off-policy methods. There are no serious ethical issues as this is basic research.

\section*{Acknowledgments}
This work was supported by the National Key Research and Development Program of China (2020AAA0106302, 2020AAA0104304), NSFC Projects (Nos. 62061136001, 62106123, 62076147, U19B2034, U1811461, U19A2081, 61972224), BNRist (BNR2023RC01004), Tsinghua Institute for Guo Qiang, and the High Performance Computing Center, Tsinghua University. J.Z was also supported by the XPlorer Prize. 



\bibliographystyle{named}
\bibliography{ref}

\newpage
\appendix

\section{Proof of Theorems}
In this section, we will provide detailed proof of theorems in the paper.
\subsection{Proof of Theorem~\ref{reused_theorem}}
\begin{proof} 
\label{reused_proof}
To simplify symbols, we set $\pi_{\mathcal{B}} = \mathcal{O}^*(\pi_0, \mathcal{B})$. First, we re-sample a dataset $\mathcal{B}'=\{\tau_i'\}_{i=1}^m\sim\hat{\Pi}$, which is independent of $\pi_{\mathcal{B}}$. Thus we have
\begin{equation}
\begin{split}
    \mathbb{E}_{\mathcal{B}\sim\hat{\Pi}}J(\pi_{\mathcal{B}})
    =& \mathbb{E}_{\mathcal{B}\sim\hat{\Pi}}\left[ \mathbb{E}_{\tau\sim \pi_{\mathcal{B}}} [R(\tau)]\right]\\
    =& \mathbb{E}_{\mathcal{B}\sim\hat{\Pi}}\left[  \frac{1}{m}\sum_{i=1}^m\left[\mathbb{E}_{\tau_i'\sim \hat{\pi}_{i}} \frac{p^{\pi_{\mathcal{B}}}(\tau_i')}{p^{\hat\pi_i}(\tau_i')}  R(\tau_i')\right]\right]\\
    =& \mathbb{E}_{\mathcal{B}\sim\hat{\Pi}}\left[ \mathbb{E}_{\mathcal{B}'\sim \hat{\Pi}} \frac{1}{m}\sum_{i=1}^m\left[\frac{p^{\pi_{\mathcal{B}}}(\tau_i')}{p^{\hat\pi_i}(\tau_i')}  R(\tau_i')\right]\right]. 
\end{split}
\end{equation}
Considering the fact that $\pi_{\mathcal{B}'}$ is the optimal over the dataset $\mathcal{B}'$, we have 
\begin{equation}
\begin{split}
    \mathbb{E}_{\mathcal{B}\sim\hat{\Pi}}J(\pi_{\mathcal{B}})
    \leq& \mathbb{E}_{\mathcal{B}\sim\hat{\Pi}}\left[ \mathbb{E}_{\mathcal{B}'\sim \hat{\Pi}} \frac{1}{m}\sum_{i=1}^m\left[\frac{p^{\pi_{\mathcal{B}'}}(\tau_i')}{p^{\hat\pi_i}(\tau_i')}  R(\tau_i')\right]\right]\\
    \overset{1}=& \mathbb{E}_{\mathcal{B}'\sim \hat{\Pi}} \frac{1}{m}\sum_{i=1}^m\left[\frac{p^{\pi_{\mathcal{B}'}}(\tau_i')}{p^{\hat\pi_i}(\tau_i')}  R(\tau_i')\right]\\
    =& \mathbb{E}_{\mathcal{B}'\sim\hat{\Pi}} \hat{J}_{\hat{\Pi}, \mathcal{B}'}(\pi_{\mathcal{B}'}) 
    = \mathbb{E}_{\mathcal{B}\sim\hat{\Pi}} \hat{J}_{\hat{\Pi}, \mathcal{B}}(\pi_{\mathcal{B}}).
\end{split}
\end{equation}
It is noted that equality 1 holds because the inner part is independent of the dataset $\mathcal{B}$. 

Thus we have proven that $\mathbb{E}_{\mathcal{B}\sim\hat{\Pi}}\hat{J}_{\hat{\Pi}, \mathcal{B}}(\pi_{\mathcal{B}}) \ge \mathbb{E}_{\mathcal{B}\sim\hat{\Pi}}J(\pi_{\mathcal{B}})$ and further $\mathbb{E}_{\mathcal{B}\sim\hat{\Pi}}\left[\epsilon_{\text{RB}}(\mathcal{O}^*, \pi_0, \mathcal{B}\right] \ge 0$. If the equality holds, then for any $\mathcal{B}, \mathcal{B}'\sim\hat{\Pi}$, we have:
\begin{equation}
\begin{split}
    &\hat{J}_{\hat{\Pi}, \mathcal{B}}(\pi_{\mathcal{B}})\\
    = &  \frac{1}{m}\sum_{i=1}^m\left[\frac{p^{\pi_{\mathcal{B}}}(\tau_i')}{p^{\hat\pi_i}(\tau_i')}  R(\tau_i')\right]
    = \frac{1}{m}\sum_{i=1}^m\left[\frac{p^{\pi_{\mathcal{B}'}}(\tau_i')}{p^{\hat\pi_i}(\tau_i')}  R(\tau_i')\right]\\
    = & \hat{J}_{\hat{\Pi}, \mathcal{B}'}(\pi_{\mathcal{B}'})
    = \mathop{\arg\max}_{\pi\in\mathcal{H}} \hat{J}_{\hat{\Pi}, \mathcal{B}'}(\pi).
\end{split}
\end{equation} 
In other words, $\mathcal{O}^*(\pi_0, \mathcal{B}) = \mathop{\arg\max}_{\pi\in\mathcal{H}} \hat{J}_{\hat{\Pi}, \mathcal{B}'}(\pi)$ holds for any $\mathcal{B}, \mathcal{B}'\sim\hat{\Pi}$. 
\end{proof} 

\subsection{Proof of Theorem~\ref{thm_one_pg}}
\begin{proof}
\label{proof_one_pg}
Since $\pi_{\theta}$ is differentiable to the parameter $\theta$, we have
\begin{equation}
\begin{split}
    \hat{J}_{\hat{\Pi}, \mathcal{B}}(\pi_{\theta'}) - \hat{J}_{\hat{\Pi}, \mathcal{B}}(\pi_{\theta}) =& \left(\theta' - \theta\right)^\top \nabla_{\theta} \hat{J}_{\hat{\Pi}, \mathcal{B}}(\pi_{\theta})\\
    +& O(\|\theta' - \theta\|^2)\\
    J(\pi_{\theta'}) - J(\pi_{\theta}) =& \left(\theta' - \theta\right)^\top \nabla_{\theta} J(\pi_{\theta})\\
    +& O(\|\theta' - \theta\|^2).
\end{split}
\end{equation}
Since $\theta' - \theta = \alpha \nabla_{\theta} \hat{J}_{\hat{\Pi}, \mathcal{B}}(\pi_{\theta})$ and $\pi_{\theta}$ is independent of $\mathcal{B}$, we can derive that
\begin{equation}
\begin{split}
    &\mathbb{E}_{\mathcal{B}\sim\hat{\Pi}} \left[ \hat{J}_{\hat{\Pi}, \mathcal{B}}(\pi_{\theta'}) - J(\pi_{\theta'}) \right]\\
    =&\mathbb{E}_{\mathcal{B}\sim\hat{\Pi}} \left[ \left(\hat{J}_{\hat{\Pi}, \mathcal{B}}(\pi_{\theta'}) - \hat{J}_{\hat{\Pi}, \mathcal{B}}(\pi_{\theta})\right) - \left(J(\pi_{\theta'}) - J(\pi_{\theta})\right) \right]\\
    =& \alpha \mathbb{E}_{\mathcal{B}\sim\hat{\Pi}} \left[\nabla_{\theta} \hat{J}_{\hat{\Pi}, \mathcal{B}}(\pi_{\theta})^\top \nabla_{\theta} \hat{J}_{\hat{\Pi}, \mathcal{B}}(\pi_{\theta}) - \nabla_{\theta} \hat{J}_{\hat{\Pi}, \mathcal{B}}(\pi_{\theta})^\top \nabla_{\theta} J(\pi_{\theta})  \right]\\
    +& O(\alpha^2).
\end{split}
\end{equation}
Since $\nabla_{\theta} \hat{J}_{\mathcal{D}}(\pi_{\theta})$ as the function of $\mathcal{D}$ is not constant and $\pi_{\theta}$ is independent of $\mathcal{B}$, we can prove
\begin{equation}
\begin{split}
    &\mathbb{E}_{\mathcal{B}\sim\hat{\Pi}} \left[ \nabla_{\theta} \hat{J}_{\hat{\Pi}, \mathcal{B}}(\pi_{\theta})^\top \nabla_{\theta} J(\pi_{\theta}) \right]\\
    =& \mathbb{E}_{\mathcal{B}\sim\hat{\Pi}} \left[ \nabla_{\theta} \hat{J}_{\hat{\Pi}, \mathcal{B}}(\pi_{\theta}) \right]^\top \nabla_{\theta} J(\pi_{\theta})\\
    =& \mathbb{E}_{\mathcal{B}\sim\hat{\Pi}} \left[ \nabla_{\theta} \hat{J}_{\hat{\Pi}, \mathcal{B}}(\pi_{\theta}) \right]^\top \mathbb{E}_{\mathcal{B}\sim\hat{\Pi}} \left[ \nabla_{\theta} \hat{J}_{\hat{\Pi}, \mathcal{B}}(\pi_{\theta}) \right]\\
    <& \mathbb{E}_{\mathcal{B}\sim\hat{\Pi}} \left[ \nabla_{\theta} \hat{J}_{\hat{\Pi}, \mathcal{B}}(\pi_{\theta})^\top \nabla_{\theta} \hat{J}_{\hat{\Pi}, \mathcal{B}}(\pi_{\theta}) \right].
\end{split}
\end{equation}
Consequently, when $\alpha>0$ is sufficient small, we have
\begin{equation}
    \mathbb{E}_{\mathcal{B}\sim\hat{\Pi}} \left[\hat{J}_{\hat{\Pi}, \mathcal{B}}(\pi_{\theta'}) - J(\pi_{\theta'})\right] > 0.
\end{equation}
\end{proof}

\subsection{Proof of Theorem~\ref{optimal_theorem}}
\begin{proof}
\label{proof_optimal_theorem}
For any $n, M\ge 1, \epsilon > 0$, we take \textbf{positive integers} $a<b$ satisfying that
\begin{equation}
    \frac{a}{b} \ge (1-\epsilon)^{\frac{1}{n+1}}
\end{equation}
Then we take \textbf{positive integers} $x, M_1, M_2$ satisfying that
\begin{equation}
\begin{split}
    x & \ge \frac{Mn}{b(1-\left(\frac{a}{b}\right)^n)} \\
    M_1 &= (b-a)x \\
    M_2 &= ax.
\end{split}
\end{equation}
Then we consider an environment with $1+M_1+M_2$ states $s_0, s_1, ..., s_{M_1+M_2}$ and $M_1+M_2$ actions $a_1, a_2, ..., a_{M_1+M_2}$. The initial state is $s_0$ and other $M_1+M_2$ states are all terminal states. Moreover, we define its transition probability and reward function as
\begin{equation}
\begin{split}
    \mathcal{P}(s_i|s_0, a_j) = 1, &i = 0, 1, ..., M_1+M_2,\\
    &j = 1, 2, ..., M_1+M_2, i = j\\
    \mathcal{P}(s_i|s_0, a_j) = 0, &i = 0, 1, ..., M_1+M_2,\\
    &j = 1, 2, ..., M_1+M_2, i \neq j\\
    \mathcal{R}(s_0, a_j) \triangleq r_j = 1, &j = 1, 2, ..., M_1,\\
    \mathcal{R}(s_0, a_j)  \triangleq r_j = 0, &j = M_1+1, M_1+2, ..., M_1+M_2.
\end{split}
\end{equation}
We take the original policy $\pi_0 = \hat{\pi}$ as the uniform policy satisfying that $\pi_0(a_j|s_0) = \hat{\pi}(a_j|s_0) = \frac{1}{M_1 + M_2}, j = 1, 2, ..., M_1+M_2$, and we take the size of the Replay Buffer $\mathcal{B}$ as $n$. Obviously, the return of the optimal policy in this environment is $J(\pi^*) = 1$. We set $p=\frac{M_2}{M_1+M_2}$ and take two action sets as $A=\{a_1,a_2,...,a_{M_1}\}, B=\{a_{M_1+1}, ..., a_{M_1+M_2}\}$.

Consequently, we have $\mathcal{B} = \{(s_0, a^i, r^i, s^i)\}_{i=1}^n$. In the first case, with the probability of $p^n$, we have $a^i\in B$ for all $i = 1, 2, ..., n$. In this case, since all $r^i = 0, i = 1, 2, ..., n$, we have $\mathcal{O}^*(\pi_0, \mathcal{B}) = \pi_0$, i.e.,
\begin{equation}
\begin{split}
    J(\mathcal{O}^*(\pi_0, \mathcal{B})) &= J(\pi_0) = \frac{M_1}{M_1+M_2} = 1-p\\
    \hat{J}_{\hat{\Pi}, \mathcal{B}}(\mathcal{O}^*(\pi_0, \mathcal{B})) &=  \frac{1}{n}\sum_{i=1}^n \frac{\mathcal{O}^*(\pi_0, \mathcal{B})(a^i|s_0)}{\hat{\pi}(a^i|s_0)} r^i = 0.
\end{split}
\end{equation}
In other cases, there exists $k(k>0)$ pairs $i_1, ..., i_k \in \{1,2,...,n\}$, $r^{i_1} = r^{i_2} = ... = r^{i_k} = 1$, we have $\sum_{t=1}^k \mathcal{O}^*(\pi_0, \mathcal{B})(a^{i_t}|s_0) = 1$, Thus we have,
\begin{equation}
\begin{split}
    J(\mathcal{O}^*(\pi_0, \mathcal{B})) &= 1\\
    \hat{J}_{\hat{\Pi}, \mathcal{B}}(\mathcal{O}^*(\pi_0, \mathcal{B})) &=  \frac{1}{n}\sum_{i=1}^n \frac{\mathcal{O}^*(\pi_0, \mathcal{B})(a^i|s_0)}{\hat{\pi}(a^i|s_0)} r^i = \frac{M_1+M_2}{n}.
\end{split}
\end{equation} 
Consequently, we have
\begin{equation}
\begin{split}
    &\mathbb{E}_{\mathcal{B}\sim\hat{\Pi}}\left[J(\mathcal{O}^*(\pi_0, \mathcal{B}))\right]\\
    = & p^n \times (1-p) + (1-p^n) \times 1\\
    = & p^n -p^{n+1} + 1-p^n\\
    = & 1 -p^{n+1},\\
    & \mathbb{E}_{\mathcal{B}\sim\hat{\Pi}}\left[\hat{J}_{\hat{\Pi}, \mathcal{B}}(\mathcal{O}^*(\pi_0, \mathcal{B}))\right]\\
    = & p^n \times 0 + (1-p^n) \times \frac{M_1+M_2}{n} \\
    = &  (1-p^n) \times \frac{M_1+M_2}{n}.
\end{split}
\end{equation}
And we have 
\begin{equation}
\begin{split}
    p &= \frac{M_2}{M_1+M_2} = \frac{ax}{ax + (b-a)x} = \frac{a}{b} \ge (1-\epsilon)^{\frac{1}{n+1}}.
\end{split}
\end{equation}
Thus we can prove that
\begin{equation}
\begin{split}
    &\mathbb{E}_{\mathcal{B}\sim\hat{\Pi}}\left[J(\mathcal{O}^*(\pi_0, \mathcal{B}))\right]\\
    = & 1 -p^{n+1}\\
    \leq & 1 - (1-\epsilon) = \epsilon,\\
    & \mathbb{E}_{\mathcal{B}\sim\hat{\Pi}}\left[\hat{J}_{\hat{\Pi}, \mathcal{B}}(\mathcal{O}^*(\pi_0, \mathcal{B}))\right]\\
    = &  (1-p^n) \times \frac{M_1+M_2}{n}\\
    = &  (1-\left(\frac{a}{b}\right)^n) \times \frac{(b-a)x + ax}{n}\\
    = &  (1-\left(\frac{a}{b}\right)^n) \times \frac{b}{n} x \ge M = M J(\pi^*).
\end{split}
\end{equation}
Thus we have proven this result.

\end{proof}


\subsection{High Probability Upper Bound for Reuse Error for Hypothesis Class of Bounded Statistical Complexity}
\label{appendix_bounded_hypo}
\begin{theorem}
Assume that, for any trajectory $\tau$, we can bound its return as $0\leq R(\tau)\leq 1$, and the hypothesis set is finite, i.e., $|\mathcal{H}|\leq \infty$. Then, for any off-policy algorithm $\mathcal{O}$ and initialized policy $\pi_0\in\mathcal{H}$, with a probability of at least $1-\delta$ over the choice of an i.i.d. training set $\mathcal{B} = \{\tau_i\}_{i=1}^m$ sampled by the same original policy $\hat{\pi}$, the following inequality holds:
\begin{equation}
\begin{split}
    |\epsilon_{\mathrm{RE}}(\mathcal{O}, \pi_0, \mathcal{B})| \leq \sqrt{\frac{\left(\max_{\pi, \tau}\frac{p^{\pi}(\tau)}{p^{\hat{\pi}}(\tau)}\right)^2}{2m}\ln\left(\frac{2|\mathcal{H}|}{\delta}\right)}.
\end{split}
\end{equation}
\end{theorem}

\begin{proof}
By Hoeffding's inequality,
we have 
\begin{equation}
\begin{split}
    & P_{\mathcal{B}\sim\hat \pi}\left(\left|\hat{J}_{\hat{\Pi}, \mathcal{B}}(\pi) - J(\pi)\right|  \ge \epsilon\right)\\
    = &P_{\mathcal{B}\sim\hat \pi}\left(\left|\frac{1}{m}\sum_{i=1}^m \frac{p^{\pi}(\tau_i)}{p^{\hat{\pi}}(\tau_i)}  R(\tau_i) - J(\pi) \right| \ge \epsilon\right)\\ 
    \leq& 2\exp\left(-\frac{2m\epsilon^2}{\left(\max_{\pi, \tau}\frac{p^{\pi}(\tau)}{p^{\hat{\pi}}(\tau)}\right)^2}\right).
\end{split}
\end{equation}
Then we have
\begin{equation}
\begin{split}
    & P_{\mathcal{B}\sim\hat \pi}\left(\exists \pi\in\mathcal{H}, \left|\hat{J}_{\hat{\Pi}, \mathcal{B}}(\pi) - J(\pi) \right| \ge \epsilon\right)\\
    \leq & \sum_{\pi\in\mathcal{H}} P_{\mathcal{B}\sim\hat \pi} \left(\forall\pi\in\mathcal{H}, \left|\hat{J}_{\hat{\Pi}, \mathcal{B}}(\pi) - J(\pi) \right|  \ge \epsilon\right)\\
    \leq& 2|\mathcal{H}| \exp\left(-\frac{2m\epsilon^2}{\left(\max_{\pi, \tau}\frac{p^{\pi}(\tau)}{p^{\hat{\pi}}(\tau)}\right)^2}\right).
\end{split}
\end{equation}
And we can prove that
\begin{equation}
\begin{split}
    & P_{\mathcal{B}\sim\hat \pi}\left( \epsilon_{\mathrm{RE}}(\mathcal{O}, \pi_0, \mathcal{B}) \ge \epsilon\right)\\
    =& P_{\mathcal{B}\sim\hat \pi}\left( \left|\hat{J}_{\hat{\Pi}, \mathcal{B}}(\mathcal{O}(\pi_0, \mathcal{B})) - J(\mathcal{O}(\pi_0, \mathcal{B})) \right|\ge \epsilon \right)\\
    \leq &  P_{\mathcal{B}\sim\hat \pi}\left(\exists \pi\in\mathcal{H}, \left|\hat{J}_{\hat{\Pi}, \mathcal{B}}(\pi) - J(\pi) \right| \ge \epsilon\right)\\
    \leq & 2|\mathcal{H}| \exp\left(-\frac{2m\epsilon^2}{\left(\max_{\pi, \tau}\frac{p^{\pi}(\tau)}{p^{\hat{\pi}}(\tau)}\right)^2}\right) \triangleq \delta.
\end{split}
\end{equation}
Thus with the probability of at least $1-\delta$, we have $\epsilon_{\mathrm{RB}}(\mathcal{O}, \pi_0, \mathcal{B}) \leq \epsilon$, and we can calculate that
\begin{equation}
\begin{split}
    \epsilon = \sqrt{\frac{\left(\max_{\pi, \tau}\frac{p^{\pi}(\tau)}{p^{\hat{\pi}}(\tau)}\right)^2}{2m}\ln\left(\frac{2|\mathcal{H}|}{\delta}\right)}.
\end{split}
\end{equation}
\end{proof}

\subsection{Proof of Theorem~\ref{main_theorem}}
\label{proof_main_theorem}
In this part, we will provide a complete proof of Theorem~\ref{main_theorem}.
First we introduce the lemma proposed in \cite{McAllester03} and provide a proof of this lemma.
\begin{lemma}[\cite{McAllester03}]
\label{lem}
Let $X$ be a real valued random variable satisfying
\begin{equation}
\begin{split}
    &P(X\ge x)\leq e^{-mf(x)},
\end{split}
\end{equation}
here $m\ge 1$. Then we have
\begin{equation}
    \mathbb{E}\left[e^{(m-1)f(X)}\right]\leq m.
\end{equation}
Specially, if $X$ satisfies $P(X\ge x)\leq e^{-2mx^2}$, we have $E\left[e^{2(m-1)X^2}\right]\leq m$.
\end{lemma}

\begin{proof}
First, we'll prove that for $\forall m\ge 1, \nu \ge 1$, $P(e^{(m-1)f(X)}\ge \nu)\leq \nu^{-m/(m-1)}$ holds. We record $F(x) = P(X\ge x)$ and have $F(X)\sim U(0,1)$. Thus we can prove that
\begin{equation}
\begin{split}
    F(x) \leq e^{-mf(x)}
    \Rightarrow & f(x)\leq \log[F(x)^{-1/m}]\\
    \Rightarrow & e^{(m-1)f(X)}\leq F(X)^{-(m-1)/m}.
\end{split}
\end{equation}
Thus we have
\begin{equation}
\begin{split}
    P(e^{(m-1)f(X)} \ge \nu)
    \leq & P(F(X)^{-(m-1)/m} \ge \nu)\\
    =& P(F(X) \leq \nu^{-m/(m-1)})\\
    = & \nu^{-m/(m-1)}.
\end{split}
\end{equation}
Consequently, we can prove that
\begin{equation}
\begin{split}
    & \mathbb{E}[e^{(m-1)f(X)}] \\
    =& \int_{0}^{\infty} P[e^{(m-1)f(X)} \ge \nu]\\
    =& \int_{0}^{1} P[e^{(m-1)f(X)} \ge \nu] + \int_{1}^{\infty} P[e^{(m-1)f(X)} \ge \nu]\\
    \leq& 1 + \int_{1}^{\infty} \nu^{-m/(m-1)}d\nu
    =1 + (m-1) = m.
\end{split}
\end{equation}
Thus we have proven Lemma~\ref{lem}.
\end{proof}
Now, based on Lemma~\ref{lem}, we will provide the complete proof of Theorem~\ref{main_theorem}
as below.
\begin{proof}
First, we will divide $\epsilon_{\text{RE}}(\mathcal{O}, \pi_0, \mathcal{B})$ into two parts as below
\begin{equation}
\begin{split}
    & |\epsilon_{\text{RE}}(\mathcal{O}, \pi_0, \mathcal{B})| \\
    =& \left|\frac{1}{m}\sum_{i=1}^m  \frac{p^{\pi}(\tau_i)}{p^{\hat{\pi}}(\tau_i)} R(\tau_i) - \mathbb{E}_{\tau\sim \pi} \left[ R(\tau)\right] \right|\\
    \leq & \left|\mathbb{E}_{\tau\sim \pi} \frac{1}{m}\sum_{i=1}^m  \left[R(\tau) - R(\tau_i)\right]\right|\\
    + & \left|\frac{1}{m}\sum_{i=1}^m  \left[1-\frac{p^{\pi}(\tau_i)}{p^{\hat{\pi}}(\tau_i)}\right] R(\tau_i)\right|\\
    \triangleq & A+B.
\end{split}
\end{equation}
here we set
\begin{equation}
\begin{split}
    A &= \left|\mathbb{E}_{\tau\sim \pi} \frac{1}{m}\sum_{i=1}^m  \left[R(\tau) - R(\tau_i)\right]\right|,\\
    B &= \left| \frac{1}{m}\sum_{i=1}^m  \left[1-\frac{p^{\pi}(\tau_i)}{p^{\hat{\pi}}(\tau_i)}\right] R(\tau_i)\right|.
\end{split}
\end{equation}
Then we will bound $A$ and $B$ respectively.

First, we extend results of PAC-Bayes Theorems~\cite{McAllester98,McAllester03} to provide a high-probability upper bound of $A$. The main difficulty is that $\pi$ may depend on $\mathcal{B}$ and $\mathbb{E}_{\tau\sim \pi} \left[ R(\tau)\right]$ is hard to control. For handling this difficulty, we propose re-sampling trick to introduce another dataset $\hat{\mathcal{B}}$ and first deform $A$ as
\begin{equation}
\begin{split}
    A &= \left|\mathbb{E}_{\tau\sim \pi} \frac{1}{m}\sum_{i=1}^m  \left[R(\tau) - R(\tau_i)\right]\right| \\
    & = \left|\frac{1}{m}\mathbb{E}_{\tau\sim \pi}\sum_{i=1}^m  \left[R(\tau) - R(\tau_i)\right]\right| \\
    &= \left|\frac{1}{m}\sum_{i=1}^m \mathbb{E}_{\hat{\tau}_i\sim \pi} \left[R(\hat{\tau}_i) - R(\tau_i)\right]\right|  \\
    & = \left|E_{\hat{\mathcal{B}}\sim\pi} \frac{1}{m}\sum_{i=1}^m  \left[R(\hat{\tau}_i) - R(\tau_i)\right]\right|,
\end{split}
\end{equation}
here $\hat{\mathcal{B}} = \{\hat{\tau}_i\}_{i=1}^m$is no longer dependent of $\pi$. Moreover, we set $\hat{A}$ as
\begin{equation}
\begin{split}
    \hat A &\triangleq (m-1)E_{\hat{\mathcal{B}}\sim\pi} \left(\frac{1}{m}\sum_{i=1}^m  \left[R(\hat{\tau}_i) - R(\tau_i)\right]\right)^2\\
    &\ge (m-1) \left(E_{\hat{\mathcal{B}}\sim\pi}\frac{1}{m}\sum_{i=1}^m  \left[R(\hat{\tau}_i) - R(\tau_i)\right]\right)^2\\
    &= (m-1)A^2.
\end{split}
\end{equation}
We set $\Delta R(\hat{\tau}_i, \tau_i) = R(\hat{\tau}_i) - R(\tau_i)$ and can further prove that
\begin{equation}
\begin{split}
    & \hat A - m  \mathrm{KL}[p^\pi(\cdot) || p^{\hat{\pi}}(\cdot)] \\
    =& (m-1)E_{\hat{\mathcal{B}}\sim\pi} \left(\frac{1}{m}\sum_{i=1}^m  \left[\Delta R(\hat{\tau}_i, \tau_i)\right]\right)^2 
    - \mathbb{E}_{\hat{\mathcal{B}} \sim \pi} \left[\log\prod_{i=1}^m \frac{p^{\pi}(\hat{\tau}_i)}{p^{\hat{\pi}}(\hat{\tau}_i)}\right] \\
    =& (m-1)E_{\hat{\mathcal{B}}\sim\pi} \left(\frac{1}{m}\sum_{i=1}^m  \left[\Delta R(\hat{\tau}_i, \tau_i)\right]\right)^2 
    + \mathbb{E}_{\hat{\mathcal{B}} \sim \pi} \left[\log \frac{p^{\hat{\pi}}(\hat{\mathcal{B}})}{p^{\pi}(\hat{\mathcal{B}})}\right] \\
    =&  \mathbb{E}_{\hat{\mathcal{B}}\sim\pi} \log \left[\frac{p^{\hat{\pi}}(\hat{\mathcal{B}})}{p^{\pi}(\hat{\mathcal{B}})} \exp\left(\frac{m-1}{m^2}\left( \sum_{i=1}^m [\Delta R(\hat{\tau}_i, \tau_i) \right)^2\right)\right]\\
    \leq& \log \mathbb{E}_{\hat{\mathcal{B}}\sim\pi} \left[\frac{p^{\hat{\pi}}(\hat{\mathcal{B}})}{p^{\pi}(\hat{\mathcal{B}})} \exp\left(\frac{m-1}{m^2}\left( \sum_{i=1}^m  \Delta R(\hat{\tau}_i, \tau_i) \right)^2\right)\right]\\
    = & \log E_{\hat{\mathcal{B}}\sim\hat{\pi}}  \left[\exp\left(((m-1)\left( \frac{1}{m}\sum_{i=1}^m  \Delta R(\hat{\tau}_i, \tau_i) \right)^2\right)\right],
\end{split}
\end{equation}
here $p^{\pi}(\hat{\mathcal{B}}) = \prod_{i=1}^m p^{\pi}(\hat{\tau}_i), p^{\hat{\pi}}(\hat{\mathcal{B}}) = \prod_{i=1}^m p^{\hat{\pi}}(\hat{\tau}_i)$. Thus we can further provide an uniform upper bound for any $\pi$ as 
\begin{equation}
\begin{split}
    &f(\mathcal{B}) \triangleq \sup_{\pi}\left[ \hat A - m  \mathrm{KL}[p^\pi(\cdot) || p^{\hat{\pi}}(\cdot)]\right]\\
    \leq & \log \mathbb{E}_{\hat{\mathcal{B}} \sim \hat{\pi}}  \left[\exp\left((m-1)\left( \frac{1}{m}\sum_{i=1}^m  \Delta R(\hat{\tau}_i, \tau_i)\right)^2\right)\right].
\end{split}
\end{equation}
We define $\hat E = \mathbb{E}_{\tau\sim p^{\hat\pi}}R(\tau)$. Since $0\leq R(\tau)\leq 1$, by Hoeffding's inequality,
we have:
\begin{equation}
\begin{split}
    &P_{\mathcal{B}\sim\hat \pi}\left(\frac{1}{m}\sum_{i=1}^m  R(\tau_i) - \hat E \ge \epsilon\right) \leq e^{-2m\epsilon^2},\\
    &P_{\mathcal{B}\sim \hat \pi}\left(\frac{1}{m}\sum_{i=1}^m  R(\tau_i) - \hat E \leq -\epsilon\right) \leq e^{-2m\epsilon^2},
\end{split}
\end{equation}
and we set
\begin{equation}
    \Delta(\mathcal{B}) = \frac{1}{m}\sum_{i=1}^m  R(\tau_i) - \hat E.
\end{equation}
Thus by Lemma~\ref{lem} we have:
\begin{equation}
    E_{\mathcal{B}\sim\hat \pi}[e^{2(m-1)\Delta(\mathcal{B})^2}] \leq m. 
\end{equation}
Similarly, we have the same result for $\hat{\mathcal{B}}$ as
\begin{equation}
    E_{\hat{\mathcal{B}}\sim\hat \pi}[e^{2(m-1)\Delta(\hat{\mathcal{B}})^2}] \leq m. 
\end{equation}
Therefore, by using Markov's inequality, we can prove that
\begin{equation}
\begin{split}
    &P_{\mathcal{B}\sim\hat \pi}[f(\mathcal{B})\ge \epsilon]
    \leq \frac{\mathbb{E}_{\mathcal{B}\sim\hat \pi} e^{f(\mathcal{B})}}{e^{\epsilon}}\\
    \leq& \frac{1}{e^{\epsilon}}\mathbb{E}_{\mathcal{B}\sim\hat \pi} \mathbb{E}_{\hat{\mathcal{B}}\sim\hat \pi} \left[\exp\left((m-1)\left( \frac{1}{m}\sum_{i=1}^m \Delta R(\hat{\tau}_i, \tau_i) \right)^2\right)\right]\\
    \leq& \frac{1}{e^{\epsilon}}\mathbb{E}_{\mathcal{B}\sim\hat \pi} \mathbb{E}_{\hat{\mathcal{B}}\sim\hat \pi} \left[\exp\left(2(m-1)\left(\left( \frac{1}{m}\sum_{i=1}^m  R(s_i, a_i) - \hat E\right)^2\right.\right.\right.\\
    &+ \left.\left.\left.\left( \hat E - \frac{1}{m}\sum_{i=1}^m R(\hat s_i, \hat a_i) \right)^2 \right)\right)\right]\\
    =& \frac{1}{e^{\epsilon}} \mathbb{E}_{\mathcal{B}\sim\hat \pi}  \left[\exp\left(2(m-1)\left(\frac{1}{m}\sum_{i=1}^m  R(\tau_i) - \hat E\right)^2\right)\right] \\ &\mathbb{E}_{\hat{\mathcal{B}}\sim\hat \pi} \left[ \exp\left( 2(m-1)\left( \hat{E} - \frac{1}{m}\sum_{i=1}^m R(\hat{\tau}_i) \right)^2\right)\right]\\
    =& \frac{1}{e^{\epsilon}} \mathbb{E}_{\mathcal{B}\sim\hat \pi}  \left[e^{2(m-1)\Delta(\mathcal{B}})^2\right]\mathbb{E}_{\hat{\mathcal{B}}\sim\hat \pi} \left[e^{2(m-1)\Delta(\hat{\mathcal{B}})^2}\right]\\
    \leq& \frac{m^2}{e^\epsilon} \triangleq \delta. 
\end{split}
\end{equation}
Thus with the probability of at least $1-\delta$, we have $f(\mathcal{B}) < \epsilon$, i.e.
\begin{equation}
\begin{split}
    A &\leq \sqrt{\frac{\hat A}{m-1}} \leq \sqrt{\frac{m  \mathrm{KL}[p^\pi(\cdot) || p^{\hat{\pi}}(\cdot)] + \epsilon}{m-1}}\\
    = &\sqrt{ \frac{m \epsilon_1 +\log\left(\frac{m^2}{\delta}\right)}{m-1}},
\end{split}
\end{equation}
which is a high probability upper bound of $A$.

For $B$, we provide a uniform upper bound as
\begin{equation}
\begin{split}
    B &= \left|\frac{1}{m}\sum_{i=1}^m  \left[1 - \frac{p^{\pi}(\tau_i)}{p^{\hat\pi}(\tau_i)} \right]R(s_i, a_i)\right|\\
    &\leq \frac{1}{m}\sum_{i=1}^m   \left|1 - \frac{p^{\pi}(\tau_i)}{p^{\hat\pi}(\tau_i)} \right|
    = \epsilon_2.
\end{split}
\end{equation}
Thus we have proven Theorem~\ref{main_theorem}.
\end{proof}

\subsection{Proof of Theorem~\ref{thm-stability}}
\begin{proof}
\label{proof-stability}
We denote $\mathcal{B}^{\setminus i}\cup \{\tau\} = \mathcal{B}\cup \{\tau\} \setminus \{\tau_i\}$ and have
\begin{equation}
\begin{split}
    &\left|\mathbb{E}_{\mathcal{B}\sim\hat{\pi}}\mathbb{E}_{\mathcal{O}}\left[\epsilon_{\text{RE}}(\mathcal{O}, \pi_0, \mathcal{B})\right]\right|\\
    = &  \left|\mathbb{E}_{\mathcal{B}\sim\hat{\pi}}\mathbb{E}_{\mathcal{O}}\left[\hat{J}_{\hat{\Pi}, \mathcal{B}}(\mathcal{O}(\pi_0, \mathcal{B})) - J(\mathcal{O}(\pi_0, \mathcal{B}))\right]\right| \\
    = &  \left|\mathbb{E}_{\mathcal{B}\sim\hat{\pi}}\mathbb{E}_{\mathcal{O}}\left[\frac{1}{n} \sum_{i=1}^n \frac{p^{\mathcal{O}(\pi_0, \mathcal{B})}(\tau_i)}{p^{\hat{\pi}}(\tau_i)}R(\tau_i)\right.\right.\\
    -&\left.\left. \mathbb{E}_{\tau\sim\hat{\pi}}\left[\frac{p^{\mathcal{O}(\pi_0, \mathcal{B})}(\tau)}{p^{\hat{\pi}}(\tau)}R(\tau)\right]\right]\right| \\
    \leq &  \left|\mathbb{E}_{\mathcal{B}\sim\hat{\pi}}\mathbb{E}_{\mathcal{O}}\left[\frac{1}{n} \sum_{i=1}^n \frac{p^{\mathcal{O}(\pi_0, \mathcal{B})}(\tau_i)}{p^{\hat{\pi}}(\tau_i)}R(\tau_i)\right.\right.\\
    -&\left.\left. \frac{1}{n} \sum_{i=1}^n \mathbb{E}_{\tau\sim\hat{\pi}}\frac{p^{\mathcal{O}(\pi_0, \mathcal{B}^{\setminus i}\cup\{\tau\})}(\tau_i)}{p^{\hat{\pi}}(\tau_i)}R(\tau_i)\right]\right| \\
    + &  \left|\mathbb{E}_{\mathcal{B}\sim\hat{\pi}}\mathbb{E}_{\mathcal{O}}\left[\frac{1}{n} \sum_{i=1}^n \mathbb{E}_{\tau\sim\hat{\pi}}\frac{p^{\mathcal{O}(\pi_0, \mathcal{B}^{\setminus i}\cup\{\tau\})}(\tau_i)}{p^{\hat{\pi}}(\tau_i)}R(\tau_i)\right.\right.\\
    -&\left.\left. \mathbb{E}_{\tau\sim\hat{\pi}}\left[\frac{p^{\mathcal{O}(\pi_0, \mathcal{B})}(\tau)}{p^{\hat{\pi}}(\tau)}R(\tau)\right]\right]\right| \\
    \leq &  \beta +\left| \mathbb{E}_{\mathcal{B}\sim\hat{\pi}}\mathbb{E}_{\mathcal{O}}\left[\frac{1}{n} \sum_{i=1}^n \mathbb{E}_{\tau\sim\hat{\pi}}\frac{p^{\mathcal{O}(\pi_0, \mathcal{B}^{\setminus i}\cup\{\tau\})}(\tau_i)}{p^{\hat{\pi}}(\tau_i)}R(\tau_i)\right.\right.\\
    -&\left.\left. \mathbb{E}_{\tau\sim\hat{\pi}}\left[\frac{p^{\mathcal{O}(\pi_0, \mathcal{B})}(\tau)}{p^{\hat{\pi}}(\tau)}R(\tau)\right]\right]\right| \\
    = &  \beta +\left|\frac{1}{n} \sum_{i=1}^n \mathbb{E}_{\mathcal{O}}\mathbb{E}_{\mathcal{B}, \tau\sim\hat{\pi}}\left[\frac{p^{\mathcal{O}(\pi_0, \mathcal{B}^{\setminus i}\cup\{\tau\})}(\tau_i)}{p^{\hat{\pi}}(\tau_i)}R(\tau_i)\right]\right.\\
    -&\left. \mathbb{E}_{\mathcal{O}}\mathbb{E}_{\mathcal{B}, \tau\sim\hat{\pi}}\left[\frac{p^{\mathcal{O}(\pi_0, \mathcal{B})}(\tau)}{p^{\hat{\pi}}(\tau)}R(\tau)\right]\right| \\
    = & \beta.
\end{split}
\end{equation}
\end{proof}

\subsection{Details and Proof of Theorem~\ref{thm_calculate_beta}}
\label{proof_calculate_beta}

In this part, we first narrate Theorem~\ref{thm_calculate_beta} in detail and then provide its proof.

In Theorem~\ref{thm_calculate_beta},
we assume that our policy $\pi_{\theta}$ is parameterized with the parameter $\theta$ and our initial parameter is $\theta_0$. We here narrate the off-policy stochastic policy gradient algorithm for any Replay Buffer $\mathcal{B}$ sampled from the origin policy $\hat{\pi}$. At each epoch $k$ ($k=0, 1, ..., T-1$), we uniformly sample a trajectory $\tau$ from $\mathcal{B}$ and update the parameter as
\begin{equation}
    \theta_{k+1} = \theta_{k} + \alpha_k \frac{\nabla_{\theta}p^{\pi_{\theta_k}}(\tau)}{p^{\hat{\pi}}(\tau)} R(\tau),
\end{equation}
and we denote $\theta_T$ as $\theta_{T,\mathcal{B}}$, which is a random variable related to $\mathcal{B}$. For proving that the off-policy stochastic policy gradient algorithm is $\beta$-uniformly stable, we need to find a constant $\beta$ and prove that for all Replay Buffer $\mathcal{B}, \mathcal{B}'$, such that $\mathcal{B}, \mathcal{B}'$ differ in at most one trajectory, we have
\begin{equation}
    \forall \tau, \theta_0,\quad \mathbb{E}_{\mathcal{O}}\left[p^{\pi_{\theta_{T, \mathcal{B}}}}(\tau) - p^{\pi_{\theta_{T, \mathcal{B}'}}}(\tau)\right]\leq\beta.
\end{equation}

\begin{proof}
We consider a more general situation that $\mathcal{B}, \mathcal{B}'$ can be totally different. To simplify the notation, we denote the parameter in epoch $k$ trained with $\mathcal{B},\mathcal{B}'$ is $\theta_k, \theta_k'$ respectively, i.e., we have
\begin{equation}
\begin{split}
    \theta_{k+1} &= \theta_k + \alpha_k \frac{\nabla_{\theta}p^{\pi_{\theta_k}}(\tau)}{p^{\hat{\pi}}(\tau)} R(\tau), \quad \tau\sim\mathcal{B}\\
    &= \theta_k + \alpha_k \frac{p^{\pi_{\theta_k}}(\tau)}{p^{\hat{\pi}}(\tau)} (\nabla_{\theta}\log p^{\pi_{\theta_k}}(\tau)) R(\tau).
\end{split}
\end{equation}
\begin{equation}
\begin{split}
    \theta_{k+1}' &= \theta_k' +  \alpha_k \frac{\nabla_{\theta}p^{\pi_{\theta_k'}}(\tau')}{p^{\hat{\pi}}(\tau')} R(\tau'),\quad \tau'\sim\mathcal{B}\\
    &= \theta_k' + \alpha_k \frac{p^{\pi_{\theta_k'}}(\tau')}{p^{\hat{\pi}}(\tau')} (\nabla_{\theta}\log p^{\pi_{\theta_k'}}(\tau')) R(\tau').
\end{split}
\end{equation}
We first bound
\begin{equation}
\begin{split}
    & \left\|\frac{p^{\pi_{\theta_k}}(\tau)}{p^{\hat{\pi}}(\tau)} (\nabla_{\theta}\log p^{\pi_{\theta_k}}(\tau)) R(\tau)\right.\\
    &- \left. \frac{p^{\pi_{\theta_k'}}(\tau')}{p^{\hat{\pi}}(\tau')} (\nabla_{\theta}\log p^{\pi_{\theta_k'}}(\tau')) R(\tau')\right\|\\
    \leq& \left\|\left(\frac{p^{\pi_{\theta_k}}(\tau)}{p^{\hat{\pi}}(\tau)}-1\right) (\nabla_{\theta}\log p^{\pi_{\theta_k}}(\tau)) R(\tau)\right\|\\
    +& \left\|(\nabla_{\theta}\log p^{\pi_{\theta_k}}(\tau)) R(\tau)\right\|
    + \left\|(\nabla_{\theta} \log p^{\pi_{\theta_k'}}(\tau')) R(\tau')\right\|\\
    +& \left\| \left(\frac{p^{\pi_{\theta_k'}}(\tau')}{p^{\hat{\pi}}(\tau')}-1\right) (\nabla_{\theta} \log p^{\pi_{\theta_k'}}(\tau')) R(\tau')\right\|\\
    \leq & \left(\left|\left(\frac{p^{\pi_{\theta_k}}(\tau)}{p^{\hat{\pi}}(\tau)}-1\right)\right| + \left| \left(\frac{p^{\pi_{\theta_k'}}(\tau')}{p^{\hat{\pi}}(\tau')}-1\right)\right| + 2\right) L.
\end{split}
\end{equation}
Set $\delta_k = \|\theta_{k} - \theta_{k}'\|$, we further have
\begin{equation}
\begin{split}
    \delta_{k+1} &= \|\theta_{k+1} - \theta_{k+1}'\|\\
    &= \left\|\theta_k - \theta_k' + \alpha_k \frac{p^{\pi_{\theta_k}}(\tau)}{p^{\hat{\pi}}(\tau)} (\nabla_{\theta}\log p^{\pi_{\theta_k}}(\tau)) R(\tau) \right.\\
    -&\left. \alpha_k \frac{p^{\pi_{\theta_k'}}(\tau')}{p^{\hat{\pi}}(\tau')} (\nabla_{\theta}\log p^{\pi_{\theta_k'}}(\tau')) R(\tau') \right\|\\
    &\leq \delta_k + \left\| \alpha_k \frac{p^{\pi_{\theta_k}}(\tau)}{p^{\hat{\pi}}(\tau)} (\nabla_{\theta}\log p^{\pi_{\theta_k}}(\tau)) R(\tau)\right. \\
    -& \left. \alpha_k \frac{p^{\pi_{\theta_k'}}(\tau')}{p^{\hat{\pi}}(\tau')} (\nabla_{\theta}\log p^{\pi_{\theta_k'}}(\tau')) R(\tau') \right\|.
\end{split}
\end{equation}
Consequently, we can prove that
\begin{equation}
\begin{split}
    & \left|\mathbb{E}[\delta_{k+1}] - \mathbb{E}[\delta_k]\right|\\
    \leq & \left\|\frac{1}{n} \sum_{i=1}^n \alpha_k \frac{p^{\pi_{\theta_k}}(\tau_i)}{p^{\hat{\pi}}(\tau_i)} (\nabla_{\theta}\log p^{\pi_{\theta_k}}(\tau_i)) R(\tau_i) \right.\\
    &-\left.  \frac{1}{n} \sum_{i=1}^n \alpha_k \frac{p^{\pi_{\theta_k'}}(\tau_i')}{p^{\hat{\pi}}(\tau_i')} (\nabla_{\theta}\log p^{\pi_{\theta_k'}}(\tau_i')) R(\tau_i')\right\|\\
    \leq & \alpha_k \frac{1}{n} \sum_{i=1}^n  \left\|  \frac{p^{\pi_{\theta_k}}(\tau_i)}{p^{\hat{\pi}}(\tau_i)} (\nabla_{\theta}\log p^{\pi_{\theta_k}}(\tau_i)) R(\tau_i) \right.\\
    & -\left. \frac{p^{\pi_{\theta_k'}}(\tau_i')}{p^{\hat{\pi}}(\tau_i')} (\nabla_{\theta}\log p^{\pi_{\theta_k'}}(\tau_i')) R(\tau_i')\right\|\\
    \leq & \alpha_k \frac{1}{n} \sum_{i=1}^n \left(\left|\left(\frac{p^{\pi_{\theta_k}}(\tau_i)}{p^{\hat{\pi}}(\tau_i)}-1\right)\right|\right.\\
    +&\left. \left| \left(\frac{p^{\pi_{\theta_k'}}(\tau_i')}{p^{\hat{\pi}}(\tau_i')}-1\right)\right| + 2\right) L\\
    = &  \left(2 + \mathbb{E}_{\tau\sim\mathcal{B}}\left|\left(\frac{p^{\pi_{\theta_k}}(\tau)}{p^{\hat{\pi}}(\tau)}-1\right)\right|\right.\\
    +& \left.\mathbb{E}_{\tau'\sim\mathcal{B}'} \left| \left(\frac{p^{\pi_{\theta_k'}}(\tau')}{p^{\hat{\pi}}(\tau')}-1\right)\right|\right) \alpha_k L_1\\
    = &  \left(2 + \mathcal{L}(\pi_{\theta_k}, \mathcal{B})  + \mathcal{L}(\pi_{\theta_k'}, \mathcal{B}) \right) \alpha_k L_1\\
    \leq & \left(2 + 2M \right) \alpha_k L_1.
\end{split}
\end{equation}
Since $\theta_0 = \theta_0'$, we have $\mathbb{E}[\delta_0] = 0$ and
\begin{equation}
    \mathbb{E}[\delta_T] \leq \left(\sum_{k=0}^{T-1} \alpha_k\right) \left(2 + 2M \right) L_1.
\end{equation}
Since $p^{\pi_{\theta}}(\tau)$ is $L_2$-Lipsticz to $\theta$, we have
\begin{equation}
\begin{split}
    \mathbb{E}_{\mathcal{O}}\left[p^{\pi_{\theta_T}}(\tau) - p^{\pi_{\theta_T'}}(\tau)\right] &\leq L_2\mathbb{E}_{\mathcal{O}}\left[\delta_T\right]\\
    \leq &\left(\sum_{k=0}^{T-1} \alpha_k\right) \left(2 + 2M \right) L_1L_2
\end{split}
\end{equation}
holds for $\forall \tau, \pi_0$.
\end{proof}

\section{Supplementary Experiment}

\subsection{Full Results on MiniGrid}
\label{appendix-minigrid}

\begin{figure*}[htbp]
\subfigure[5$\times$5, 30]{
\begin{minipage}[t]{0.33\linewidth}
\centering
\includegraphics[height=3.5cm,width=6.0cm]{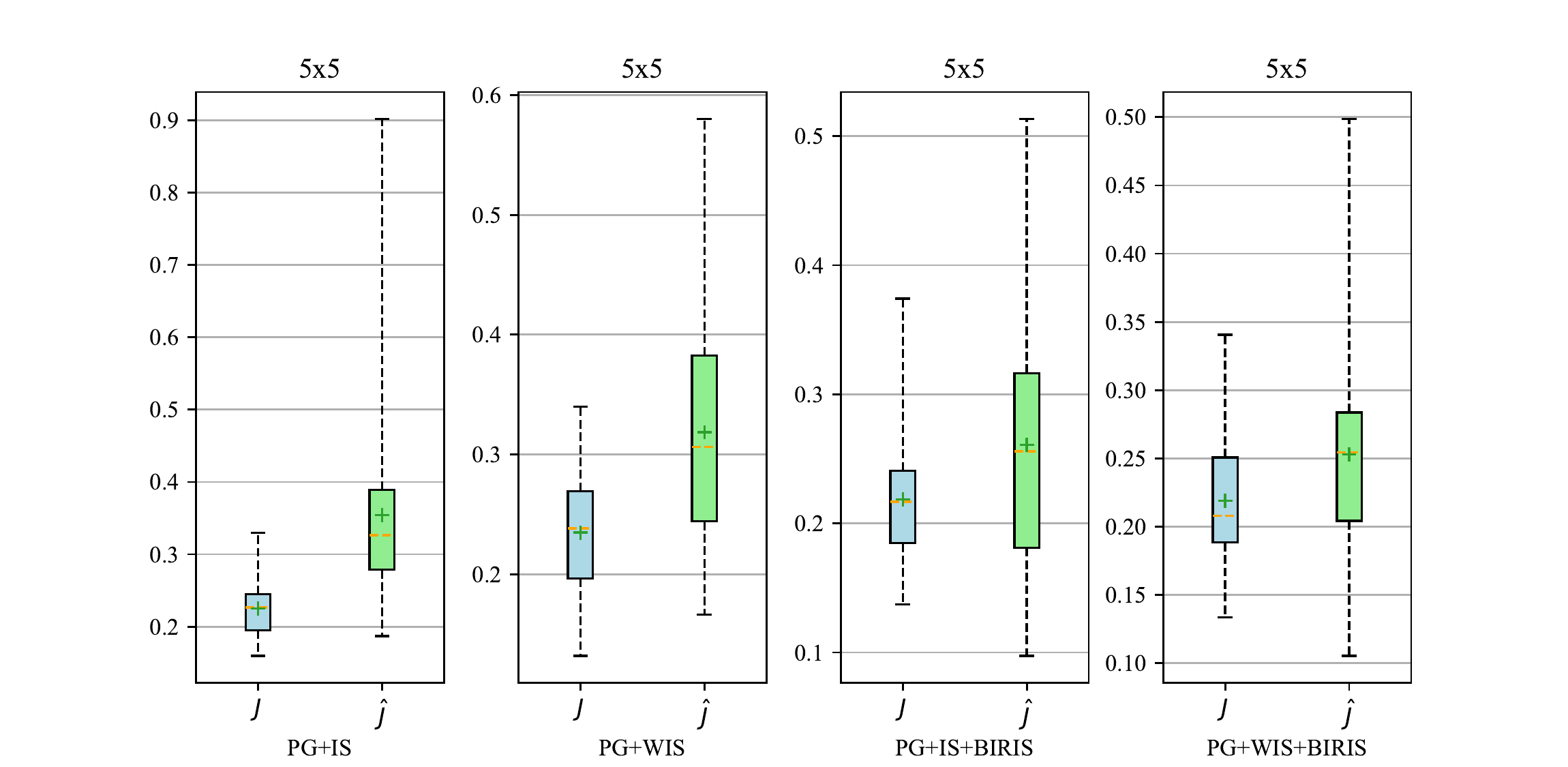}
\end{minipage}}
\subfigure[5$\times$5, 40]{
\begin{minipage}[t]{0.33\linewidth}
\centering
\includegraphics[height=3.5cm,width=6.0cm]{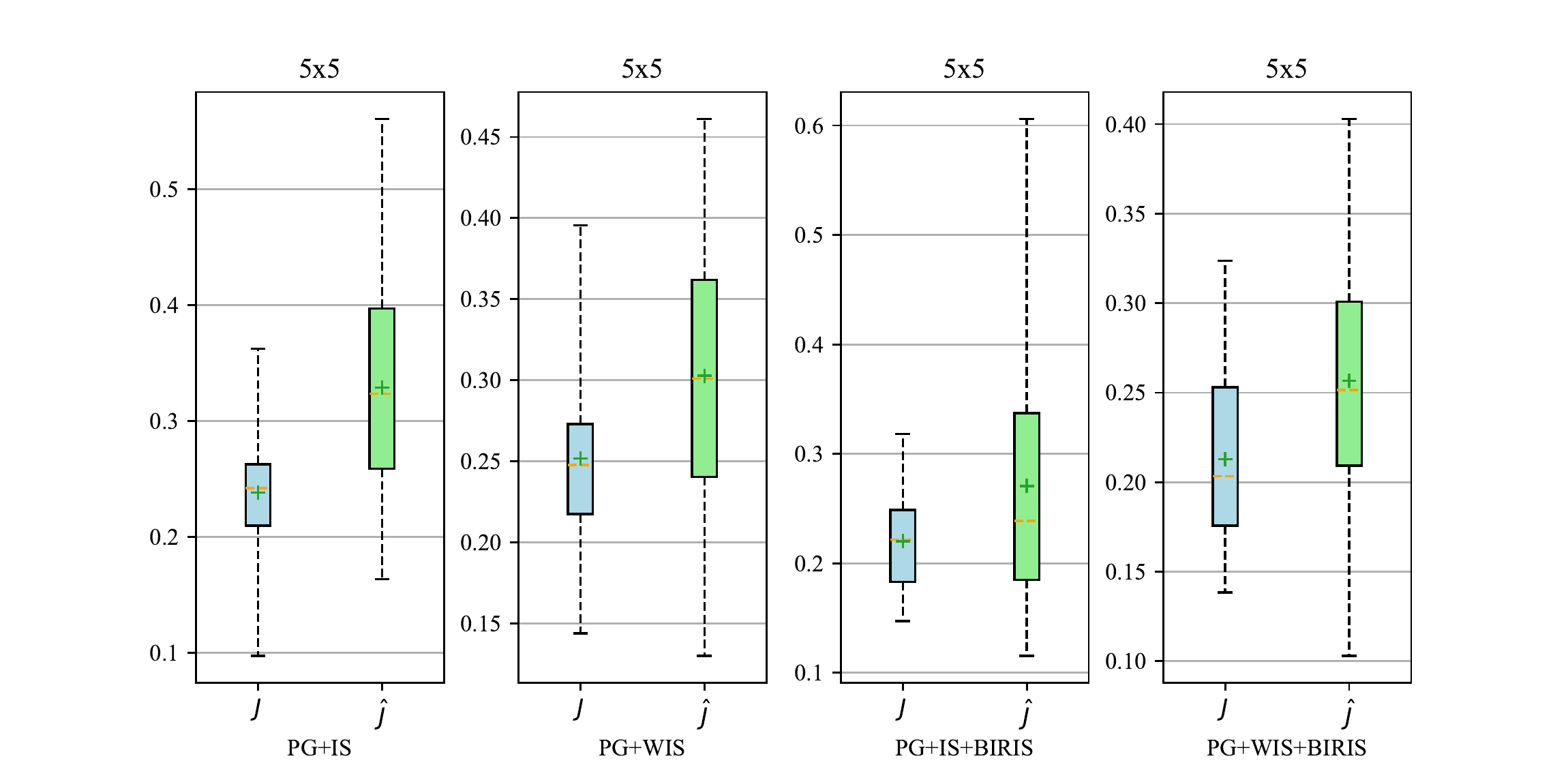}
\end{minipage}}
\subfigure[5$\times$5, 50]{
\begin{minipage}[t]{0.33\linewidth}
\centering
\includegraphics[height=3.5cm,width=6.0cm]{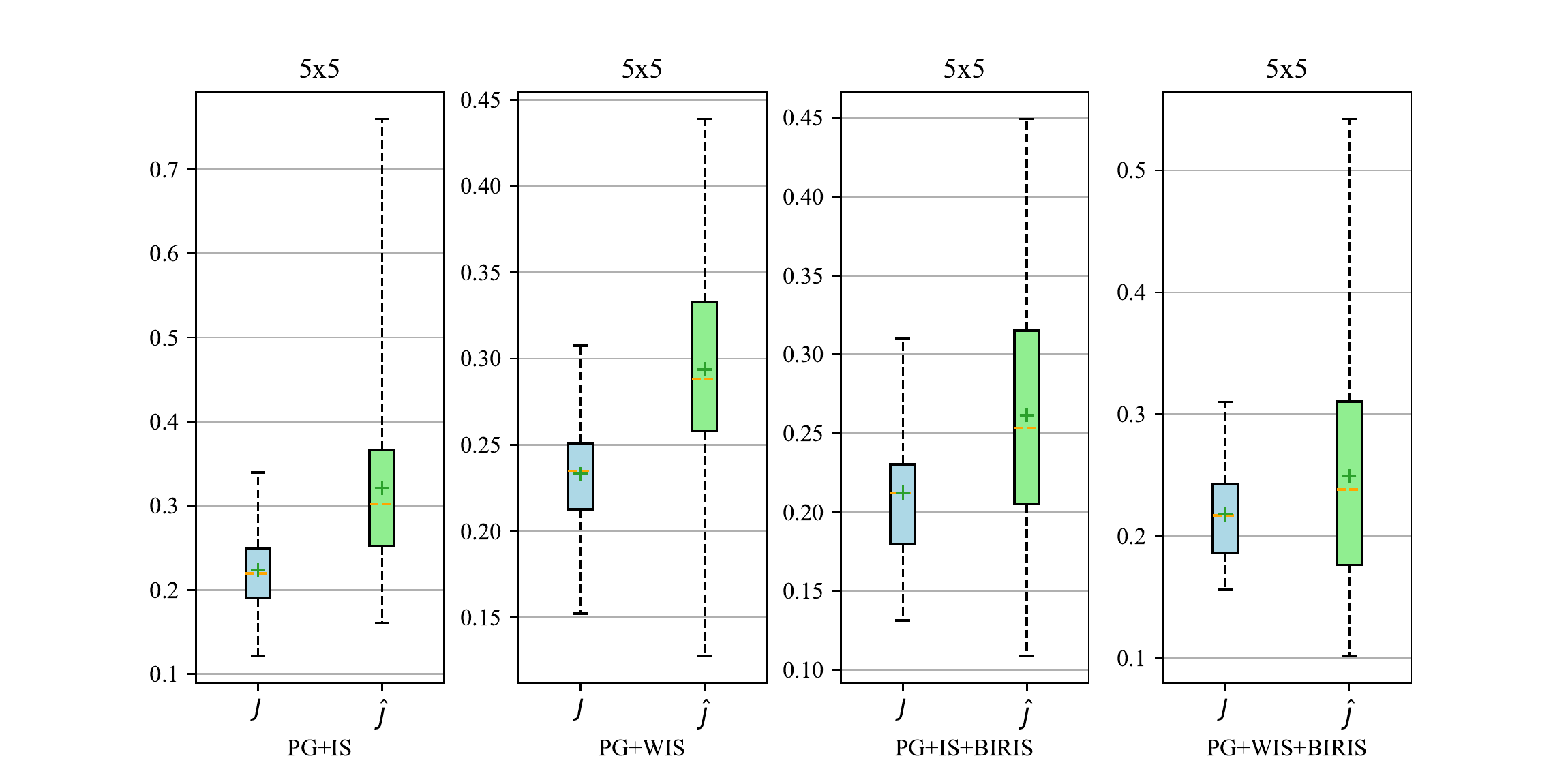}
\end{minipage}}
\end{figure*}

\begin{figure*}[htbp]
\subfigure[5$\times$5-Random, 30]{
\begin{minipage}[t]{0.33\linewidth}
\centering
\includegraphics[height=3.5cm,width=6.0cm]{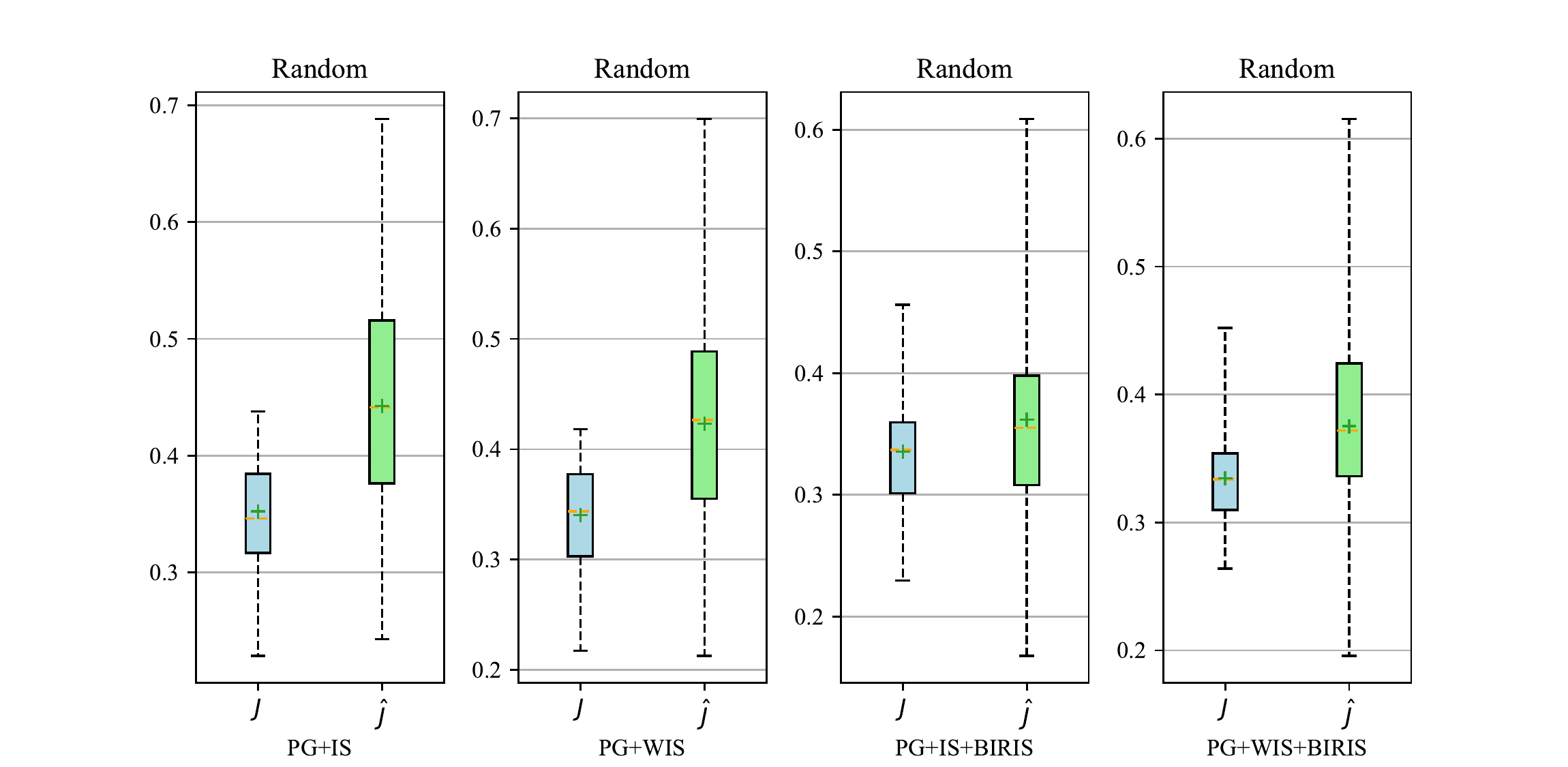}
\end{minipage}}
\subfigure[5$\times$5-Random, 40]{
\begin{minipage}[t]{0.33\linewidth}
\centering
\includegraphics[height=3.5cm,width=6.0cm]{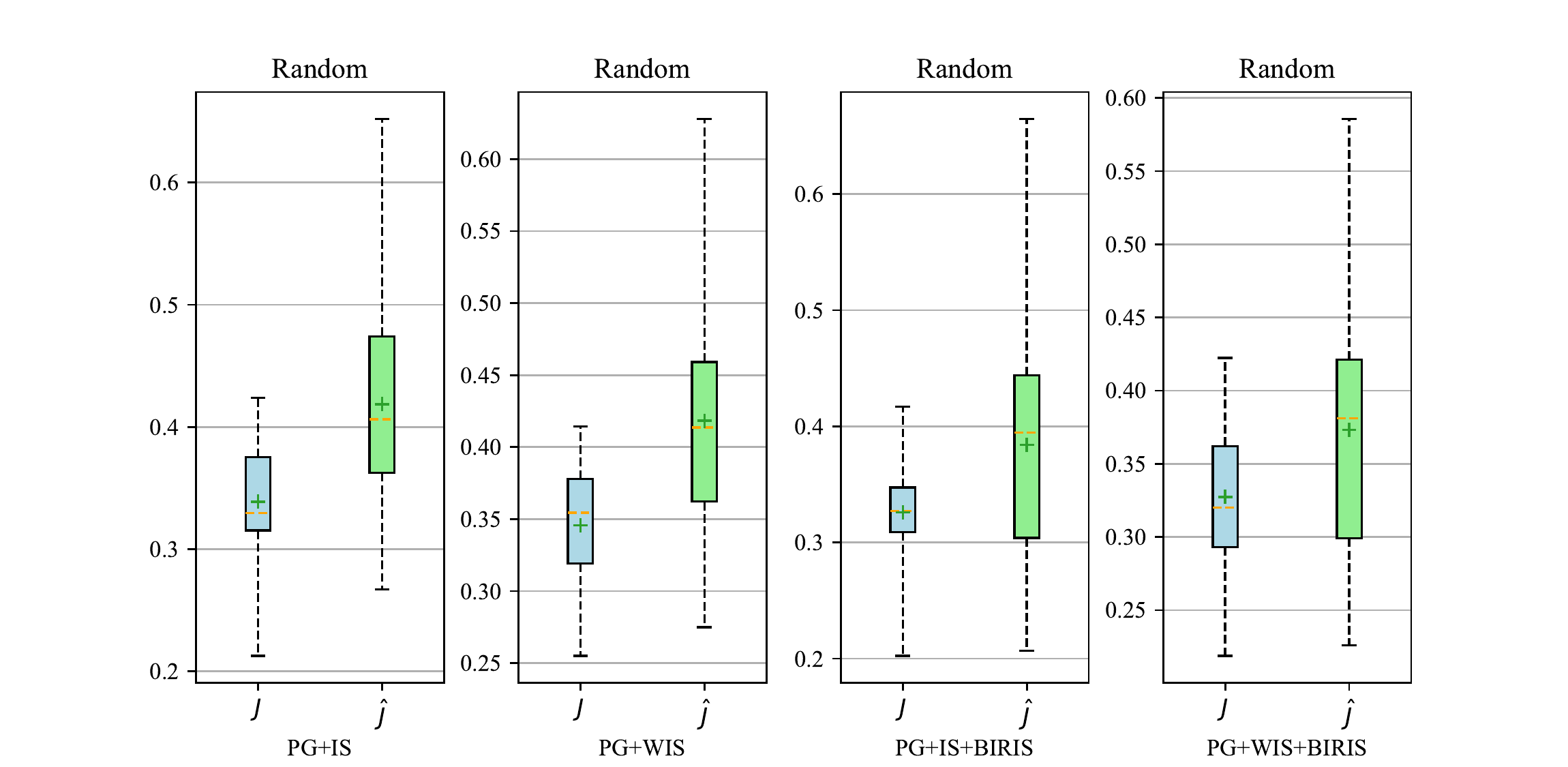}
\end{minipage}}
\subfigure[5$\times$5-random, 50]{
\begin{minipage}[t]{0.33\linewidth}
\centering
\includegraphics[height=3.5cm,width=6.0cm]{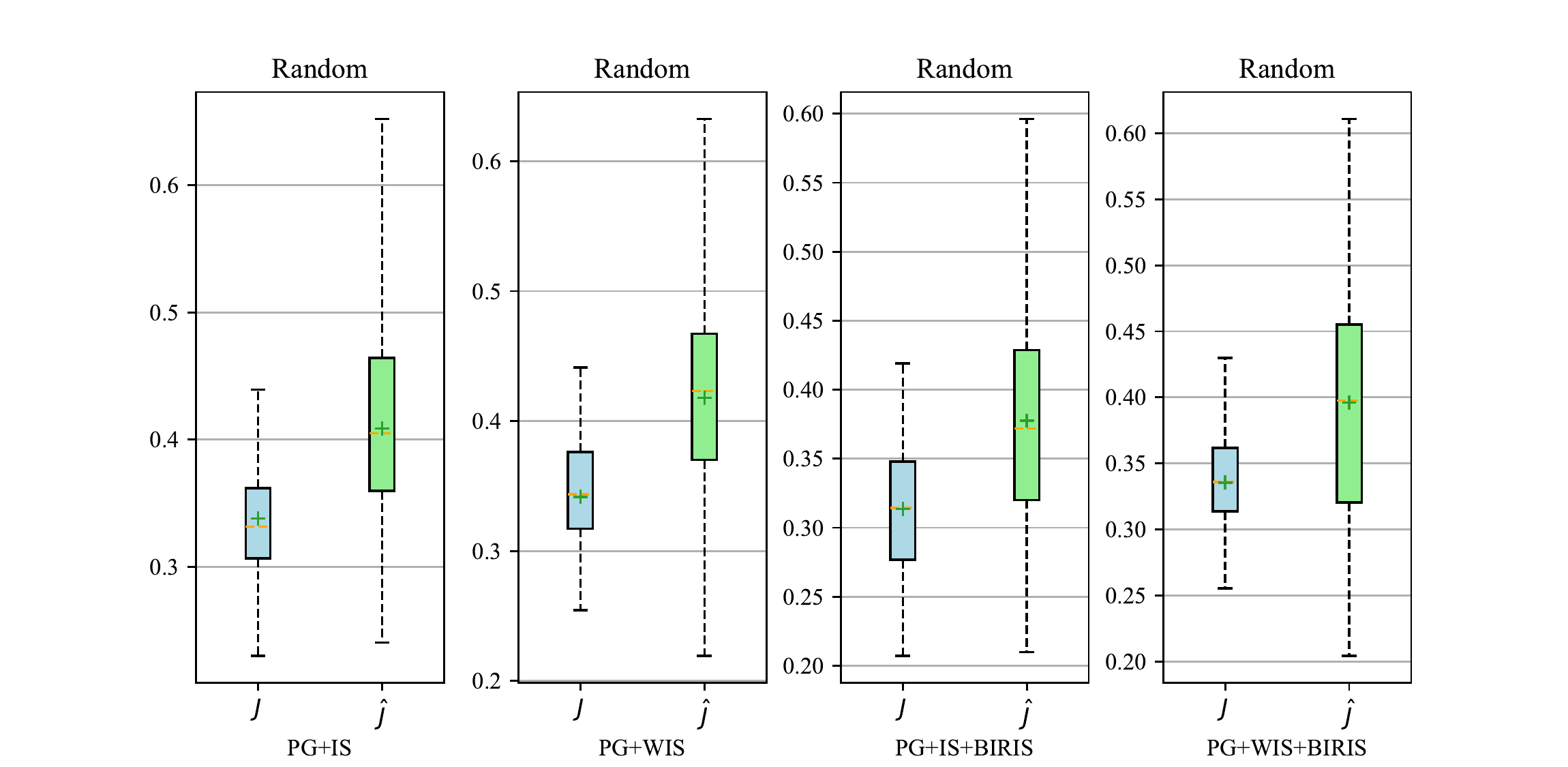}
\end{minipage}}
\end{figure*}

\begin{figure*}[htbp]
\subfigure[6$\times$6, 30]{
\begin{minipage}[t]{0.33\linewidth}
\centering
\includegraphics[height=3.5cm,width=6.0cm]{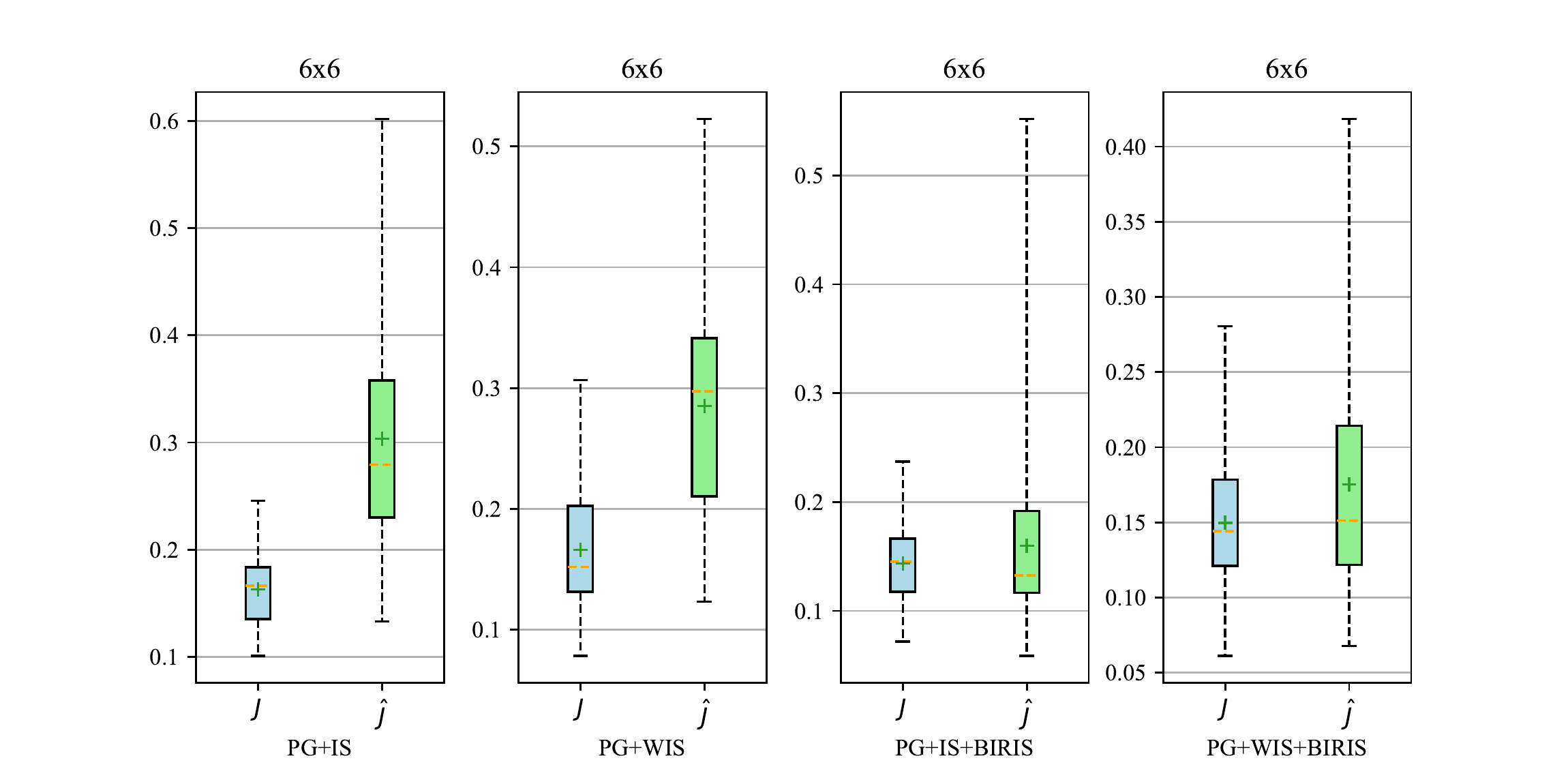}
\end{minipage}}
\subfigure[6$\times$6, 40]{
\begin{minipage}[t]{0.33\linewidth}
\centering
\includegraphics[height=3.5cm,width=6.0cm]{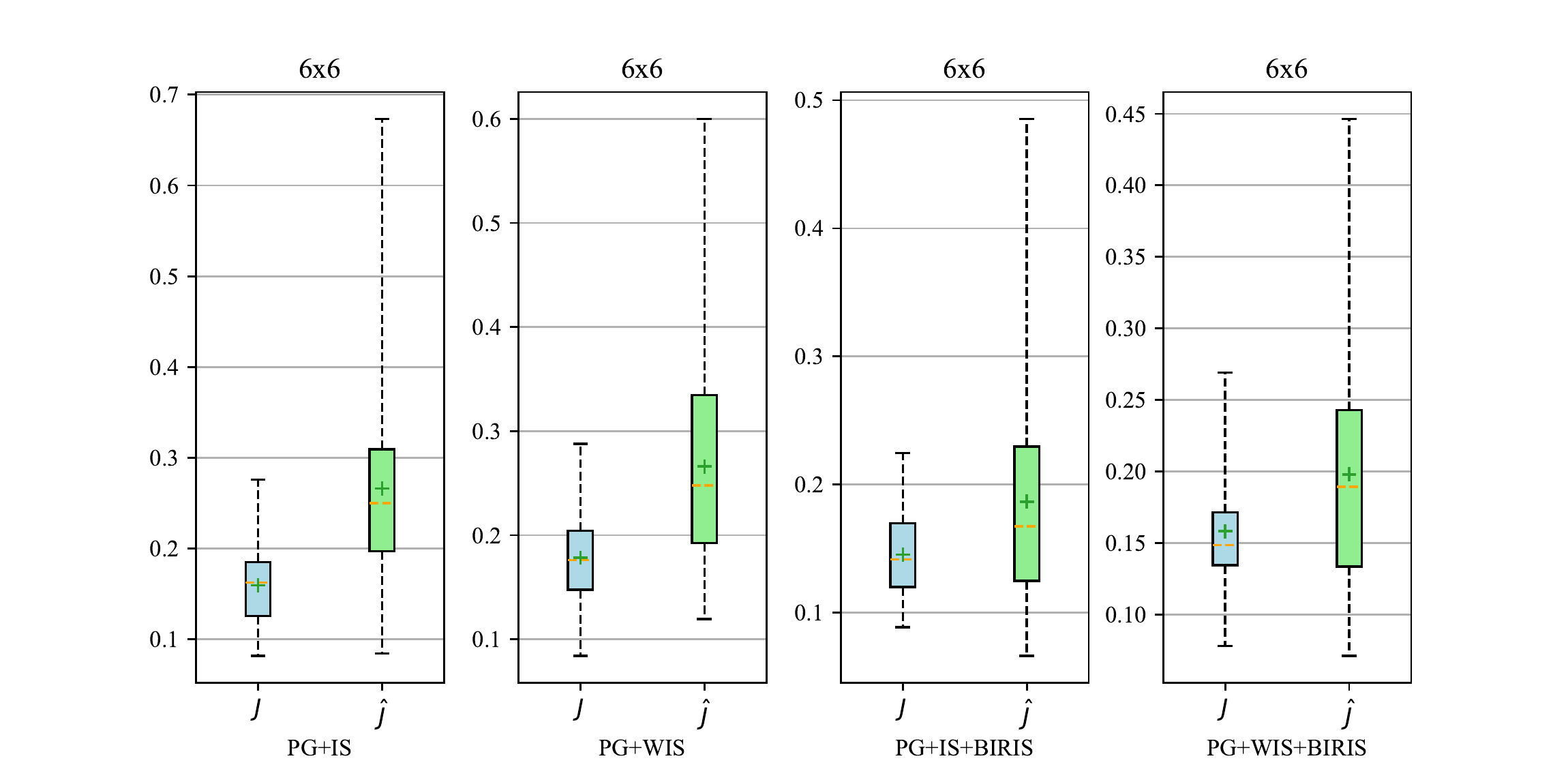}
\end{minipage}}
\subfigure[6$\times$6, 50]{
\begin{minipage}[t]{0.33\linewidth}
\centering
\includegraphics[height=3.5cm,width=6.0cm]{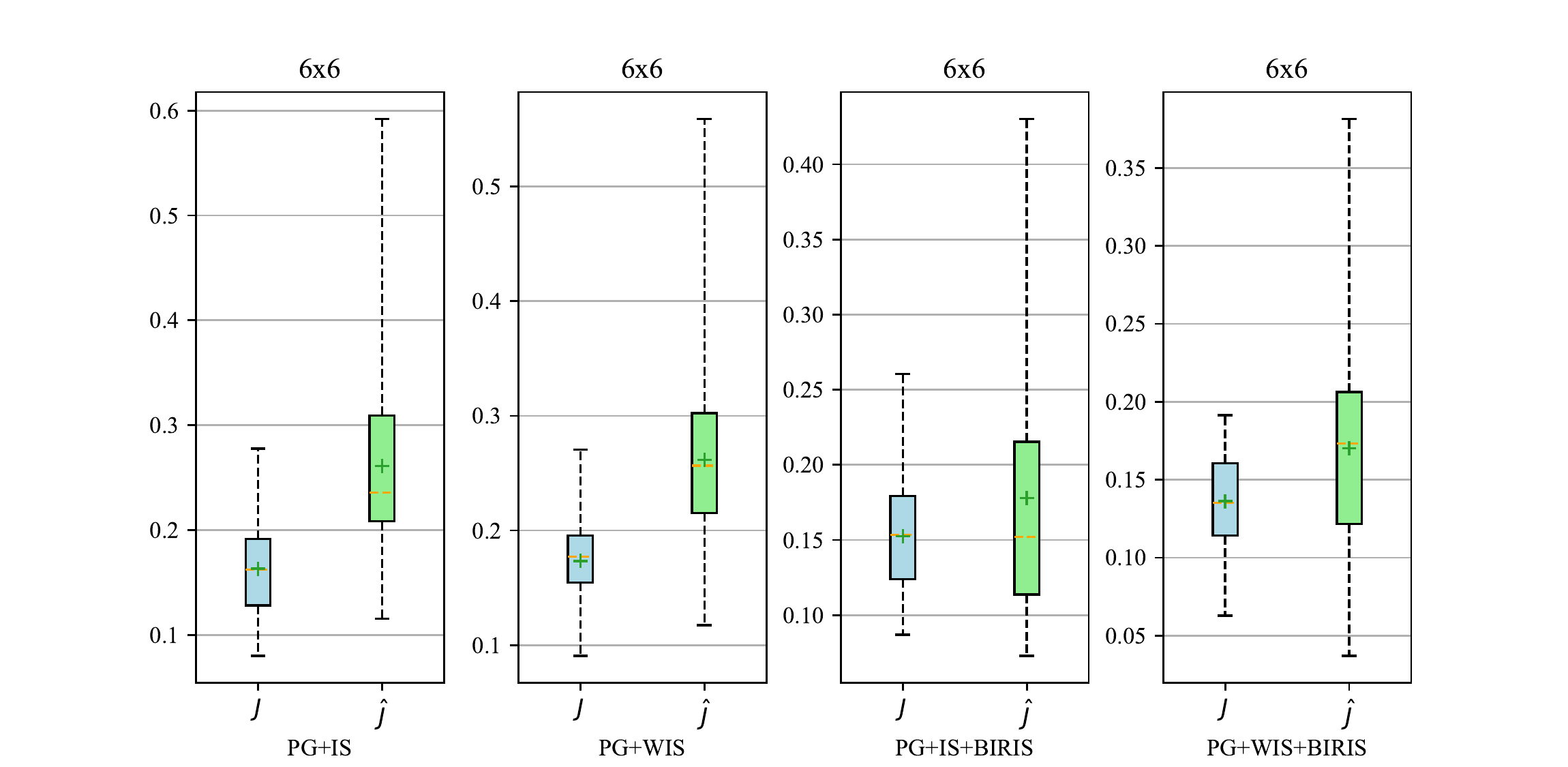}
\end{minipage}}
\end{figure*}

\begin{figure*}[htbp]
\subfigure[6$\times$6-random, 30]{
\begin{minipage}[t]{0.33\linewidth}
\centering
\includegraphics[height=3.5cm,width=6.0cm]{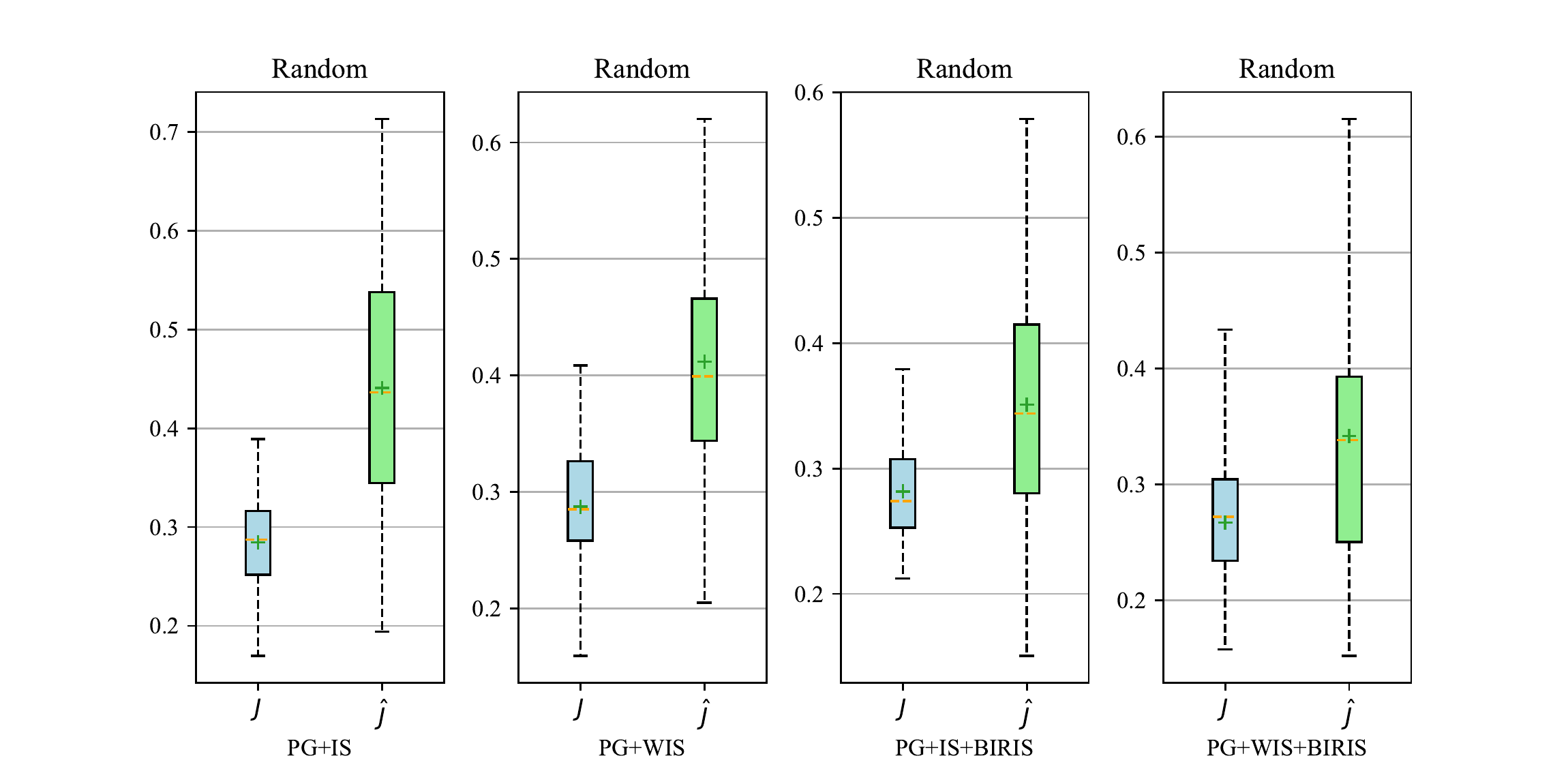}
\end{minipage}}
\subfigure[6$\times$6-random, 40]{
\begin{minipage}[t]{0.33\linewidth}
\centering
\includegraphics[height=3.5cm,width=6.0cm]{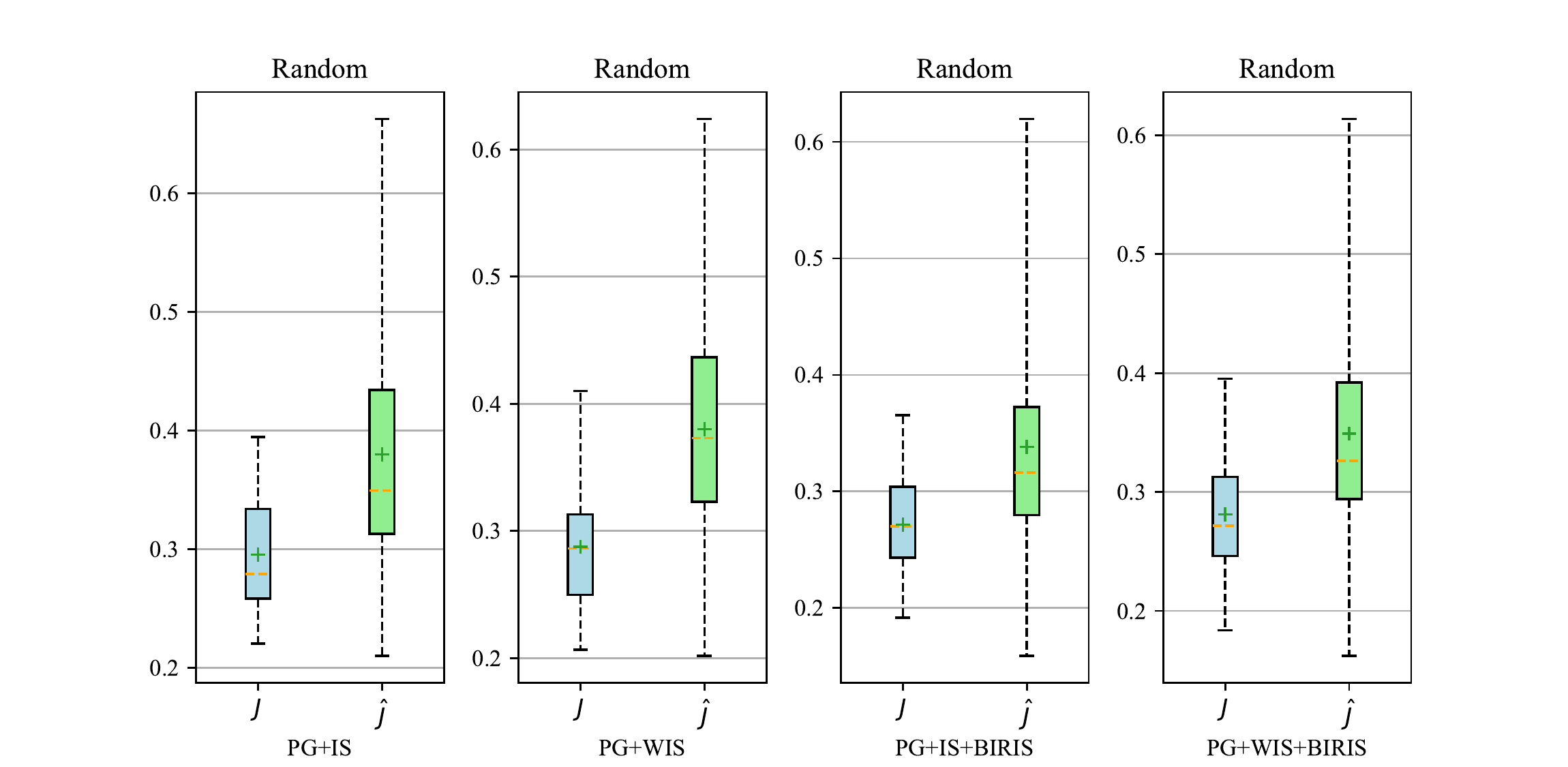}
\end{minipage}}
\subfigure[6$\times$6-random, 50]{
\begin{minipage}[t]{0.33\linewidth}
\centering
\includegraphics[height=3.5cm,width=6.0cm]{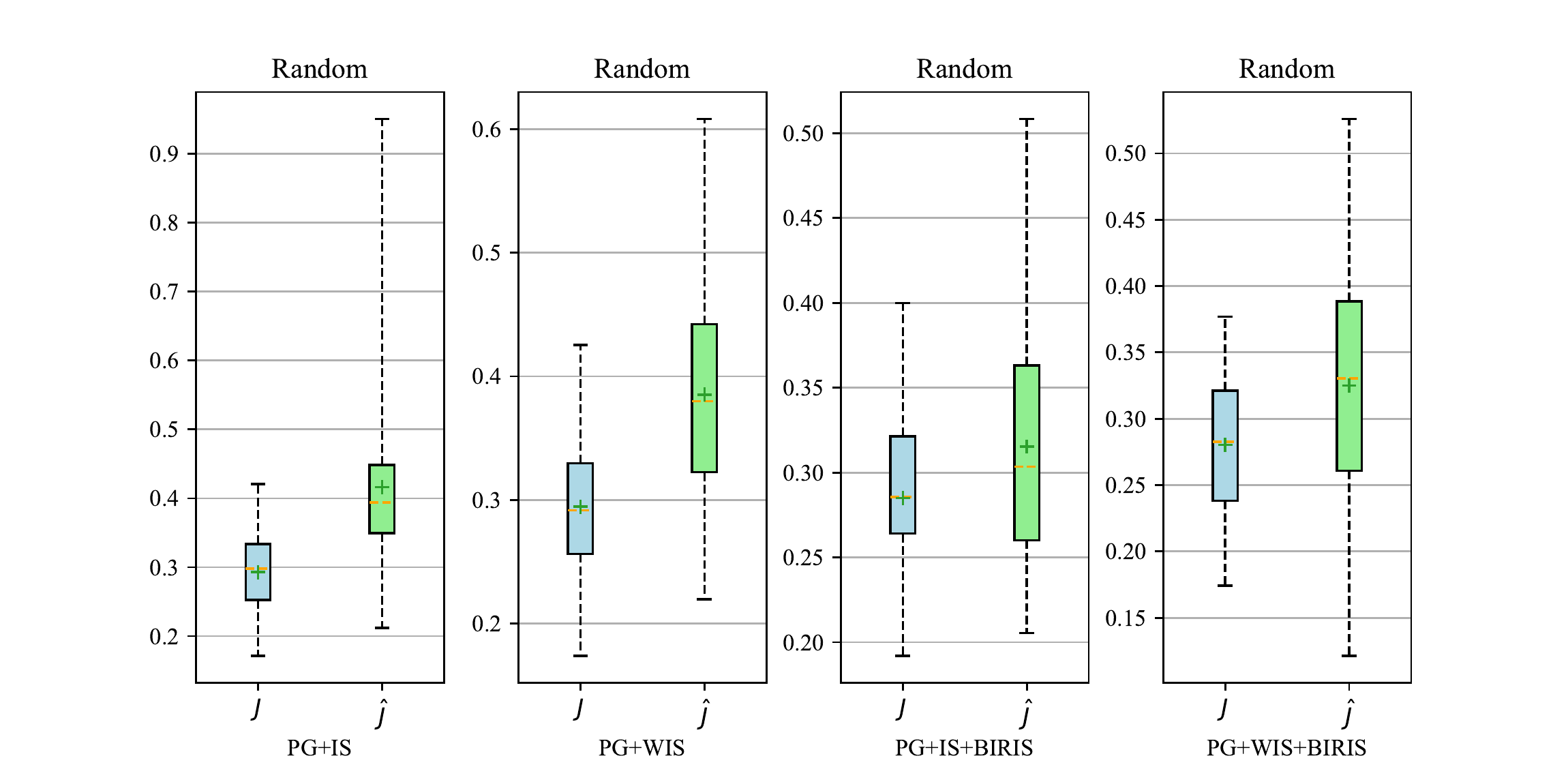}
\end{minipage}}
\end{figure*}


\begin{figure*}[htbp]
\subfigure[8$\times$8, 30]{
\begin{minipage}[t]{0.33\linewidth}
\centering
\includegraphics[height=3.5cm,width=6.0cm]{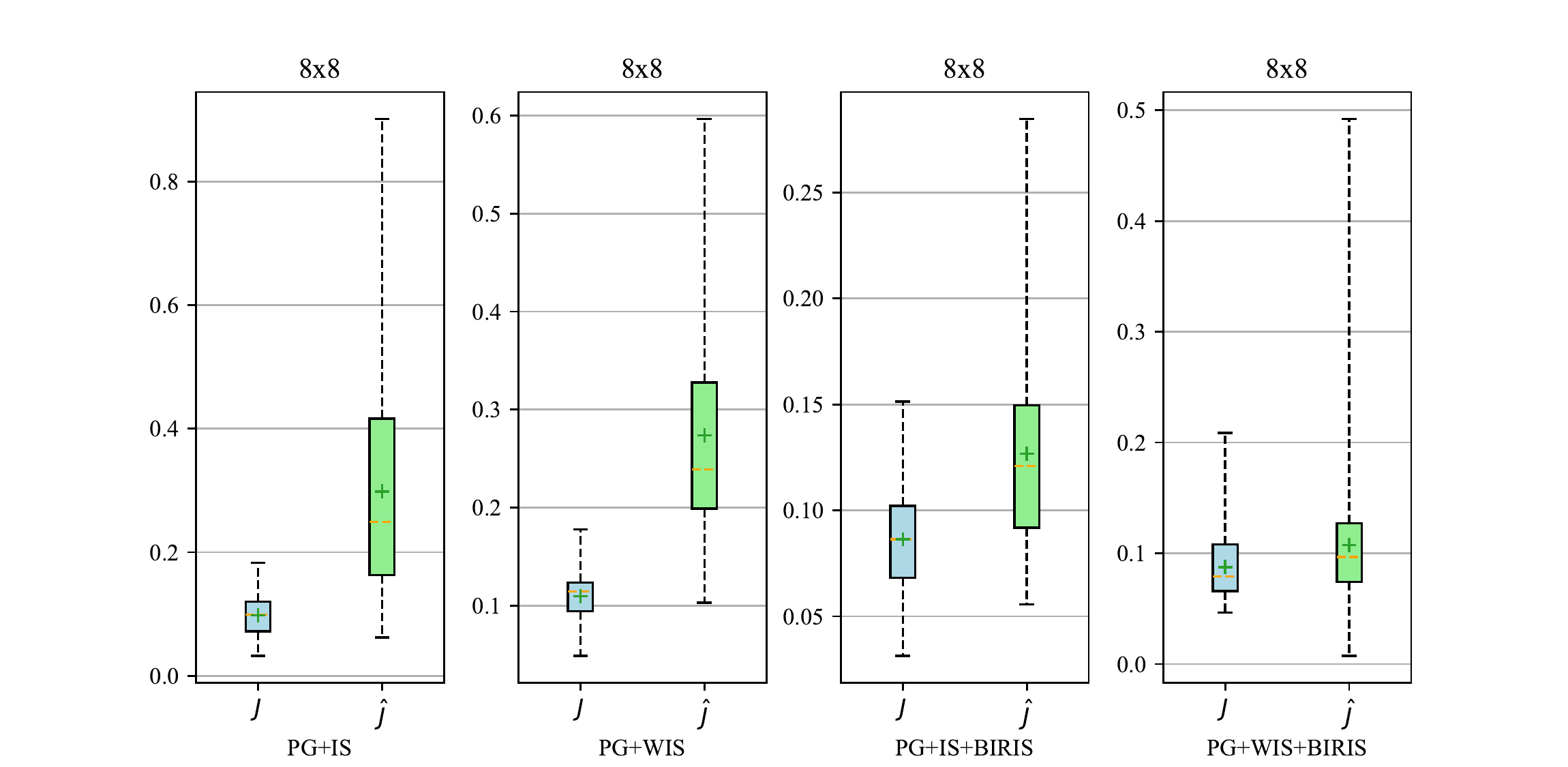}
\end{minipage}}
\subfigure[8$\times$8, 40]{
\begin{minipage}[t]{0.33\linewidth}
\centering
\includegraphics[height=3.5cm,width=6.0cm]{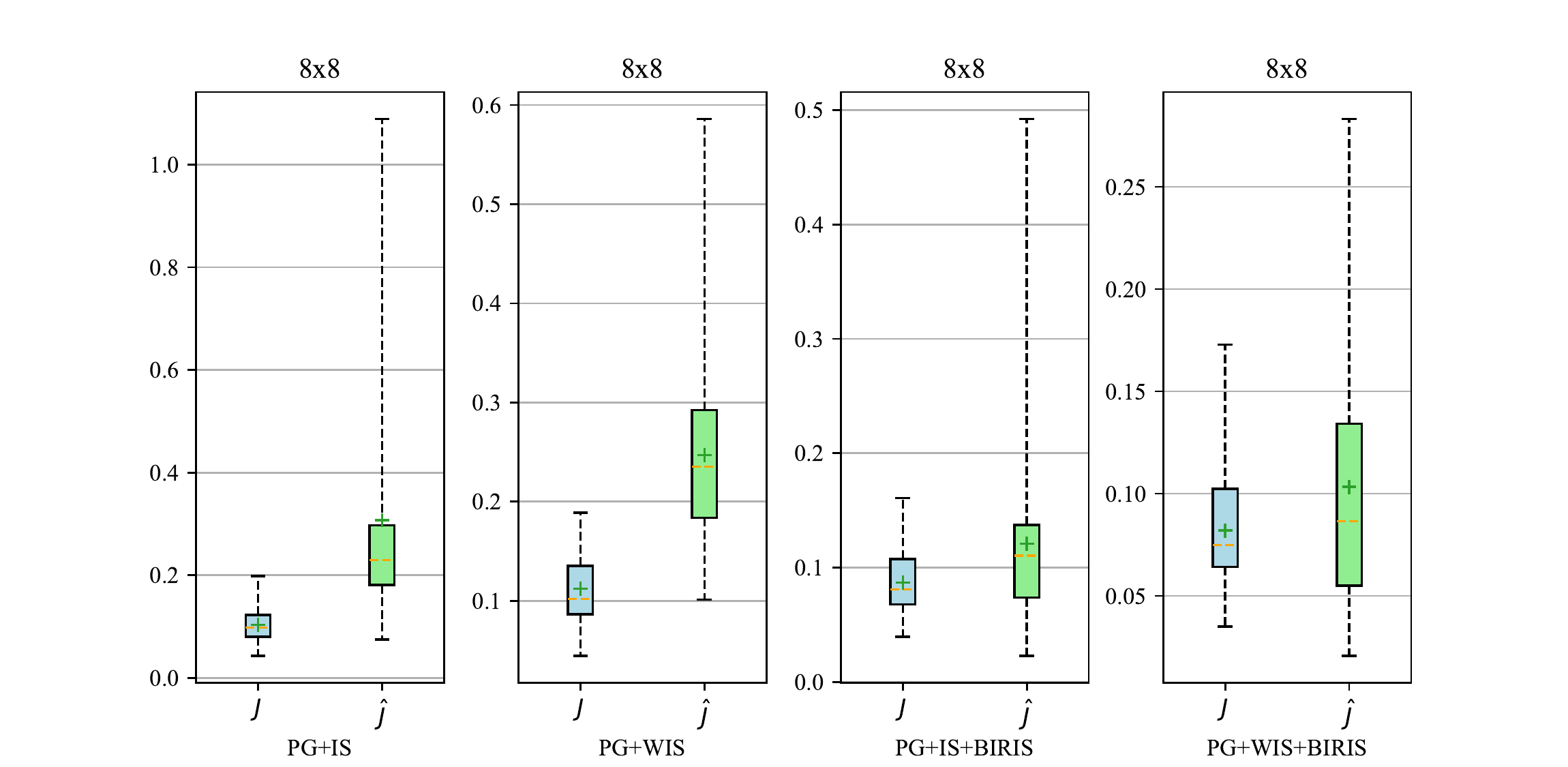}
\end{minipage}}
\subfigure[8$\times$8, 50]{
\begin{minipage}[t]{0.33\linewidth}
\centering
\includegraphics[height=3.5cm,width=6.0cm]{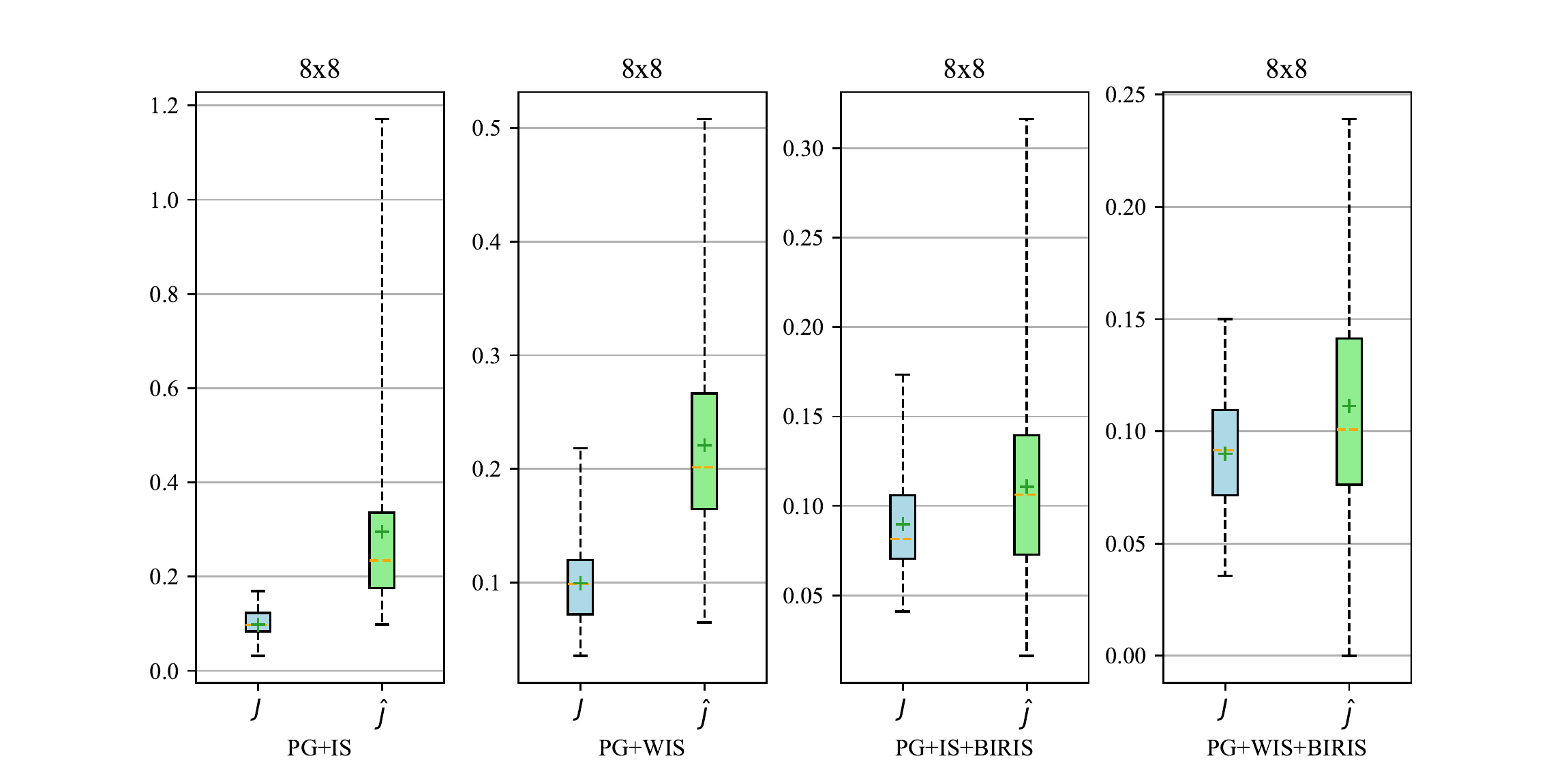}
\end{minipage}}
\end{figure*}

\begin{figure*}[htbp]
\subfigure[16$\times$16, 30]{
\begin{minipage}[t]{0.33\linewidth}
\centering
\includegraphics[height=3.5cm,width=6.0cm]{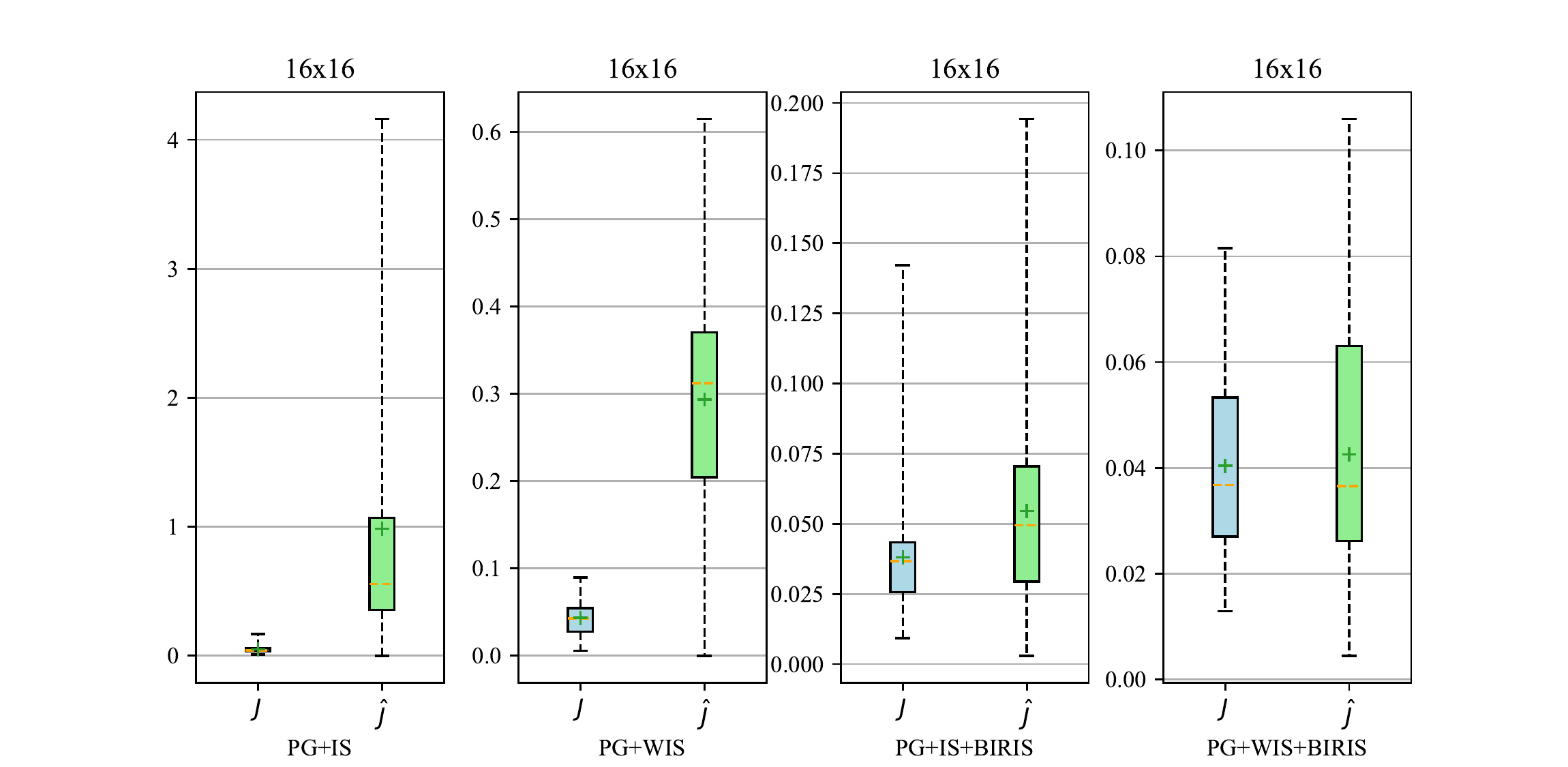}
\end{minipage}}
\subfigure[16$\times$16, 40]{
\begin{minipage}[t]{0.33\linewidth}
\centering
\includegraphics[height=3.5cm,width=6.0cm]{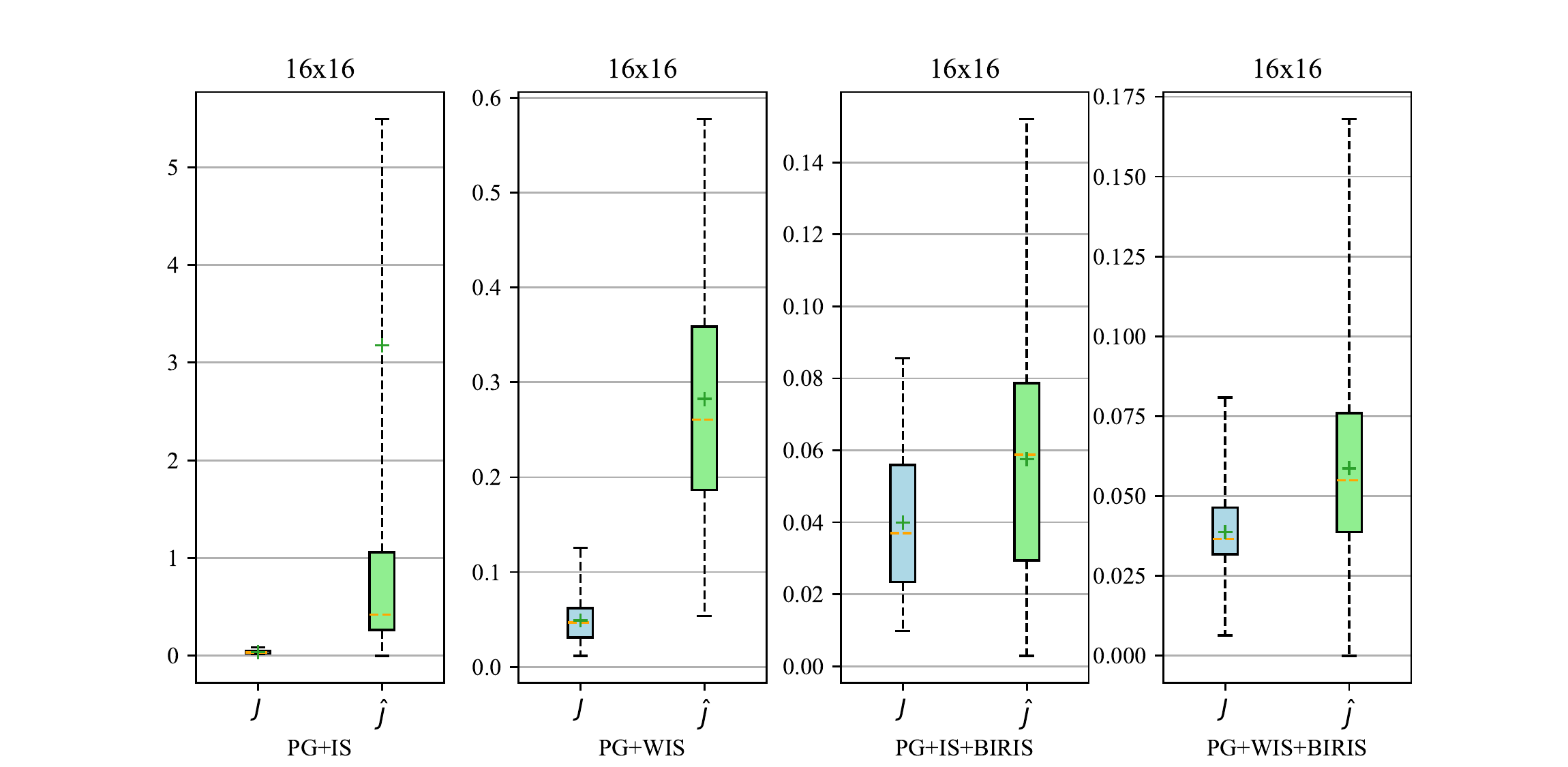}
\end{minipage}}
\subfigure[16$\times$16, 50]{
\begin{minipage}[t]{0.33\linewidth}
\centering
\includegraphics[height=3.5cm,width=6.0cm]{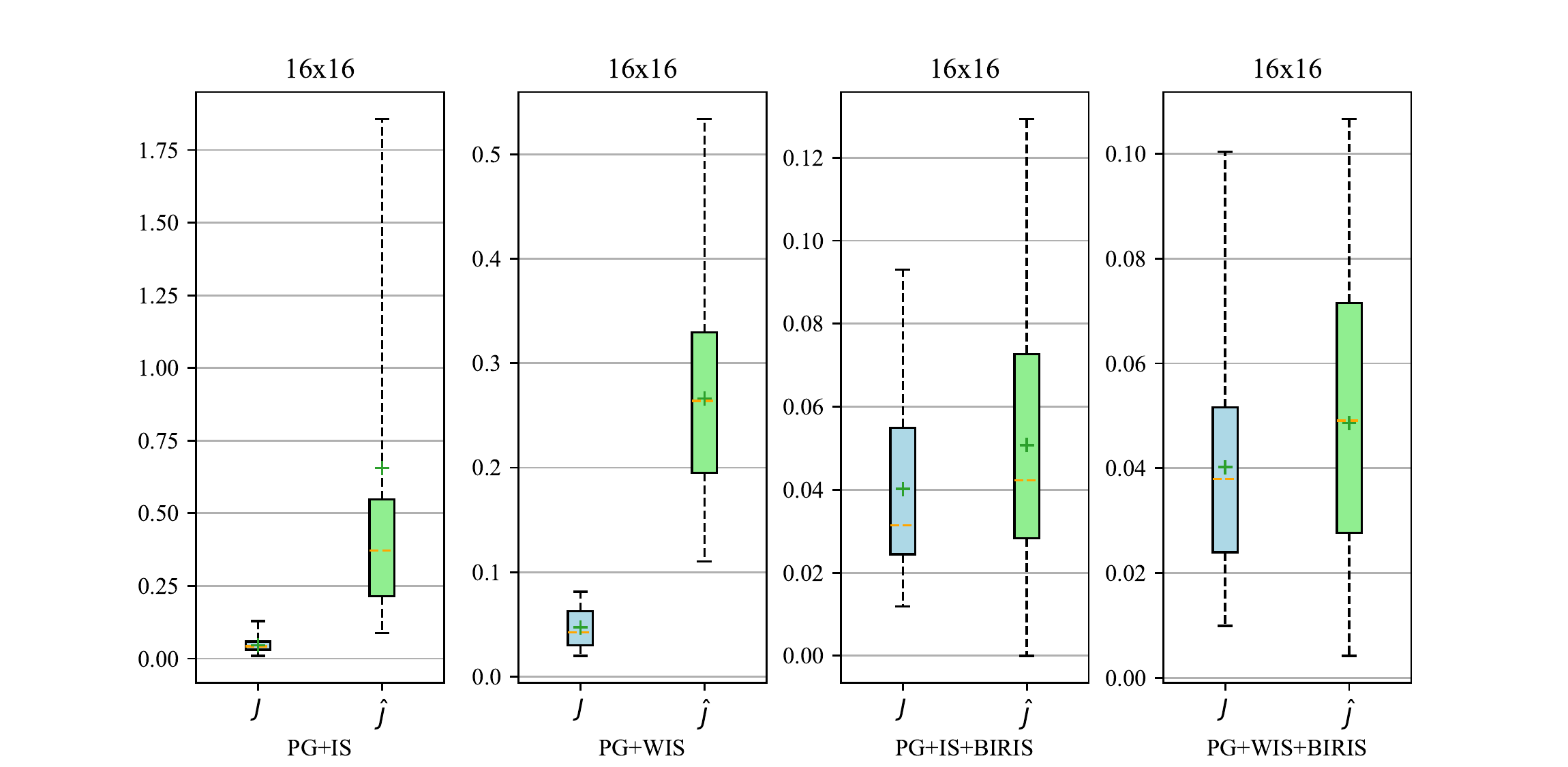}
\end{minipage}}
\end{figure*}

\newpage

\subsection{Ablation Study on MuJoCo}
\label{appendix_supp_expe}
In this part, we report supplementary experimental results about the influence of clipping and the influence of the hyperparameter $\alpha$.
\begin{figure}[h]
  \centering
  \subfigure{\includegraphics[width=0.35\columnwidth]{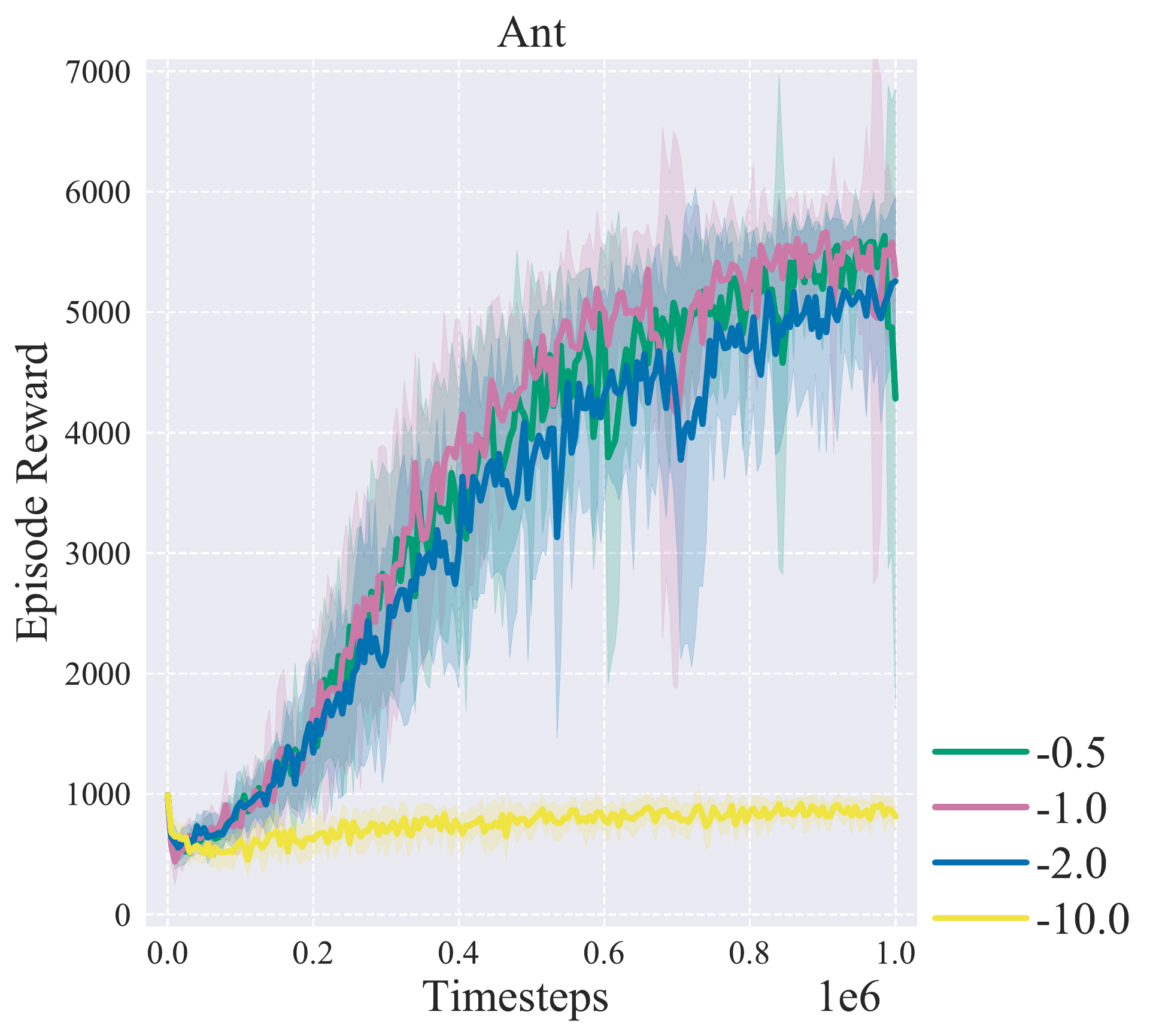}}
  \subfigure{\includegraphics[width=0.35\columnwidth]{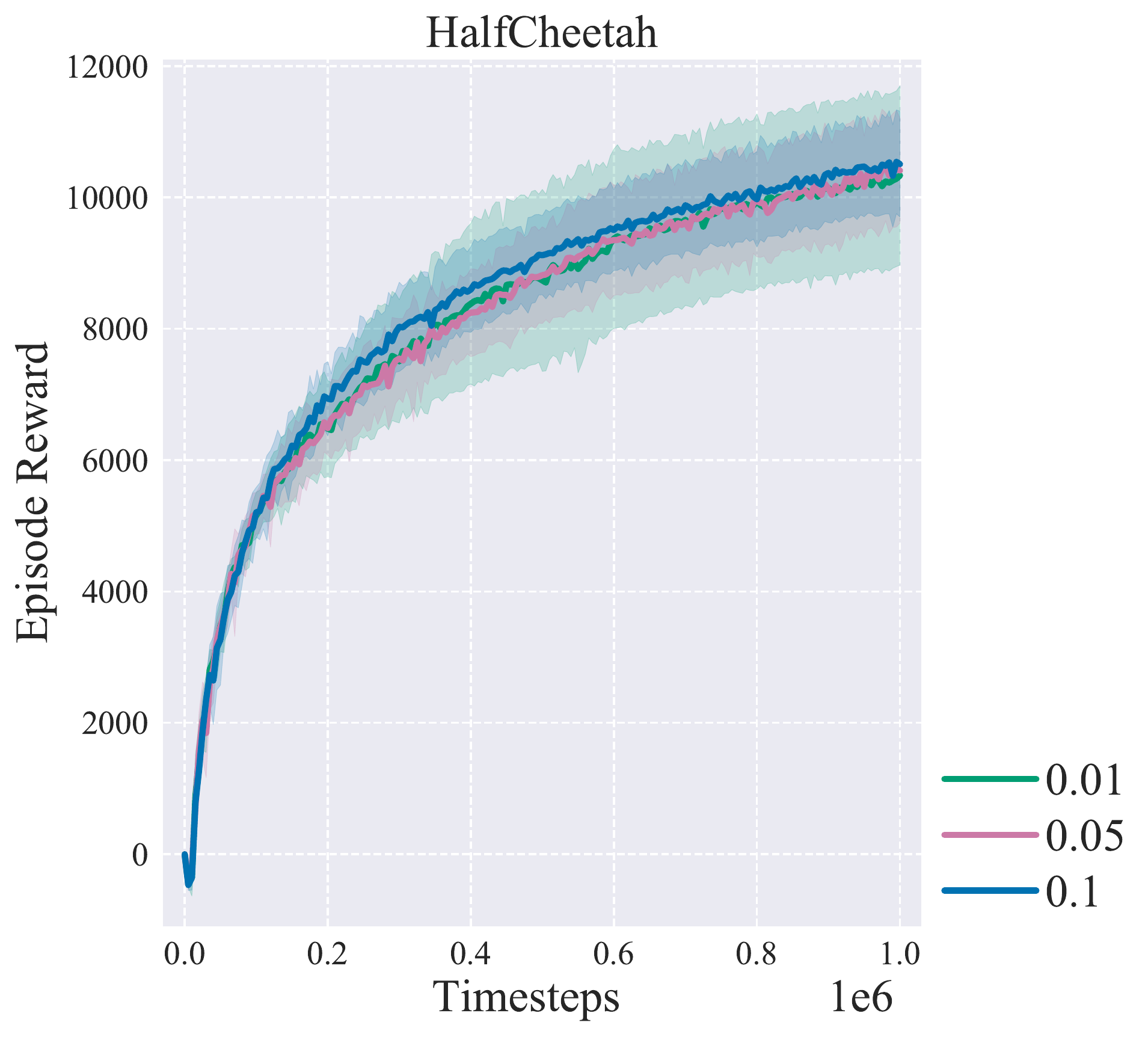}}
  \caption{Results of ablation study. The left one reports the influence of clipping. The right one reports the influence of the hyperparameter $\alpha$}
  \label{abla_fig}
\end{figure}

\textbf{Ablation study on the influence of clipping.} We clip the value of $\log(\hat{\pi}(a|s))$ in SAC+BIRIS because too low a value may cause numerical instability. To evaluate the effect of clipping, we consider SAC+BIRIS in the Ant task and set different values of $\beta$ ($-0.5, -1.0, -2.0, -10.0$). As shown in the left part of Fig.~\ref{abla_fig}, the performance of SAC+BIRIS is similar when $\beta$ is within a certain range (e.g. $\beta=-0.5, -1.0, -2.0$). However, if the value of $\beta$ is too small (e.g. $\beta=-10.0$), the performance of SAC+BIRIS is extremely poor because the weights of the state-action pair $(s,a)$ with small $\hat{\pi}(a|s)$ will be extremely large. Consequently, the choice of clipping helps our SAC+BIRIS become more stable by avoiding numerical instability, which is significant in policy training. 

\textbf{Ablation study on the influence of the hyperparameter $\alpha$.} The choice of the hyperparameter $\alpha$ is important for BIRIS because it is a trade-off between standard RL loss and the BIRIS regular term. Since $\alpha$ has similar effects in SAC+BIRIS and TD3+BIRIS, we consider TD3+BIRIS in the HalfCheetah task to evaluate the influence of different values of $\alpha$ ($0.01, 0.05, 0.1$). As shown in the right part of Fig.~\ref{abla_fig}, the performance of TD3+BIRIS is almost similar with different values of $\alpha$, which means that BIRIS is stable when the value of $\alpha$ is within the appropriate range. 

\section{A Concrete Example for Reuse Bias}
\label{appendix_example}
We consider an environment with the fixed initial state $s_0$ and we can take an action $a\in\mathbb{R}$. After taking the action $a$, we arrive the state $s_1$, get the return $\mathcal{R}(s_0,a)=1$, and end our trajectory. Consequently, in this environment, every policy owns the same cumulative return as 1. Now we assume $\hat{\pi}$ is the behavior policy and take $m$ trajectories sampled from $\hat{\pi}$ to form $\mathcal{B} = \{(s_0,a_i,s_1,r_i=1)\}_{i=1}^m$. We denote the trained policy $\mathcal{O}(\pi_0, \mathcal{B})$ as $\pi_{\mathcal{B}}$ for simplicity and assume that our training algorithm $\mathcal{O}$ obtains the following policy
\begin{equation}
\begin{split}
    \pi_{\mathcal{B}}(a|s_0) = \left\{
    \begin{aligned}
         & \hat{\pi}(a|s_0),  &\forall a\neq a_1,a_2,...,a_m\\
         &  0, & a = a_1,a_2,...,a_m
    \end{aligned}
    \right.
\end{split}
\end{equation}
Since $m$ is finite, we would like to emphasize $\pi_{\mathcal{B}}$ is still a possible policy since $\int_{\mathbb{R}}\pi_{\mathcal{B}}(a|s_0) da = 1$ still holds. Moreover, we can naturally calculate that $J(\pi_{\mathcal{B}}) = 1$ and $\hat{J}_{\hat{\pi},\mathcal{B}}(\pi_{\mathcal{B}}) =\frac{1}{m}\sum_{i=1}^m \frac{\pi_{\mathcal{B}}(a_i|s_0)}{\hat{\pi}(a_i|s_0)}\mathcal{R}(s_0,a_i)= 0$. Thus we have $|\epsilon_{\mathrm{RE}}(\mathcal{O},\hat{\pi},\mathcal{B})| = 1$ holds for any $\mathcal{B}$. This example shows that directly increasing the size of the replay buffer cannot eliminate the Reuse Bias if our hypothesis set includes all policies.

\section{BIRIS-based Algorithms}
\label{appendix_biris}
The core of BIRIS-based methods is to calculate $\frac{\pi(a|s)}{\hat{\pi}(a|s)}$ for any state-action pair $(s,a)\in\mathcal{B}$ and further calculate $\mathcal{L}_{\text{BR}}(\pi, \mathcal{B})$

\textbf{SAC+BIRIS.} For handling continuous action space, SAC~\cite{sac} chooses stochastic policies with reparameterization. Thus we can directly calculate that
\begin{equation}
\begin{split}
    \frac{\pi(a|s)}{\hat{\pi}(a|s)} 
    &= \exp\left(\log\left(\pi(a|s)\right) - \log\left(\hat{\pi}(a|s)\right)\right).
\end{split}
\end{equation}
However, this ratio may be unstable when $\hat{\pi}(a|s)$ is too small. Therefore, to avoid the potential numerical instability,
we clip its value with a hyperparameter $\beta$ in SAC+BIRIS as $
    \exp\left(\log\left(\pi(a|s)\right) - \mathrm{clip}\left( \log\left(\hat{\pi}(a|s), \beta, 0\right)\right)\right)$.

\textbf{TD3+BIRIS.} 
For algorithms with deterministic policies like TD3~\cite{td3}, our policy is a fixed mapping from $\mathcal{S}$ to $\mathcal{A}$, i.e., $\pi:\mathcal{S}\rightarrow\mathcal{A}$. To improve the diversity of actions, we always sample action from a Gaussian distribution $\mathcal{N}(\pi(s), \Sigma)$. 
Here, $\Sigma$ is fixed and we can calculate that
\begin{equation}
\begin{split}
    \frac{\pi(a|s)}{\hat{\pi}(a|s)} 
    = & \exp\left(\frac{1}{2}\|a - \hat{\pi}(s)\|^2-\frac{1}{2}\|a - \pi(s))\|^2\right).
\end{split}
\end{equation}


\end{document}